\documentclass[letterpaper]{article} 
\usepackage{aaai2027}  
\usepackage[hyphens]{url}  
\usepackage{graphicx} 
\usepackage{natbib}  
\usepackage{caption} 
\usepackage{amsmath}          
\usepackage{amssymb}
\usepackage{amsthm}           
\usepackage{multirow}
\usepackage{subcaption}
\usepackage{algorithm}
\usepackage{algorithmic}

\DeclareMathOperator{\op}{op}

\newtheorem{theorem}{Theorem}
\newtheorem{lemma}[theorem]{Lemma}
\newtheorem{definition}[theorem]{Definition}
\newtheorem{assumption}[theorem]{Assumption}
\newtheorem{remark}[theorem]{Remark}   
\usepackage{newfloat}
\usepackage{listings}
\DeclareCaptionStyle{ruled}{labelfont=normalfont,labelsep=colon,strut=off} 
\floatstyle{ruled}
\newfloat{listing}{tb}{lst}{}
\floatname{listing}{Listing}

\usepackage{booktabs}

\title{Socialized Division and Collaboration: \\Rethinking Class-Incremental Learning under Optimization Conflicts}

\author {
    Xinjie Yao,\textsuperscript{\rm 1}
    Zhihe Fan,\textsuperscript{\rm 2}
    Yunqi Zhu,\textsuperscript{\rm 3}
    Jiaqi Zhou,\textsuperscript{\rm 4}
    Dengyu Zhao,\textsuperscript{\rm 4}
    Zhoupeng Guo,\textsuperscript{\rm 5}
    Yan Fan,\textsuperscript{\rm 6}
    Guosong Jiang,\textsuperscript{\rm 4}
    Pengfei Zhu\textsuperscript{\rm 4,5}\thanks{Corresponding author.}
}
\affiliations{
    \textsuperscript{\rm 1} Faculty of Information Engineering and Automation, Kunming University of Science and Technology\\
    \textsuperscript{\rm 2} School of Sports Training, Tianjin University of Sport\\
    \textsuperscript{\rm 3} School of Computer Science and Engineering, University of New South Wales\\
    \textsuperscript{\rm 4} School of Artificial Intelligence, Tianjin University\\
    \textsuperscript{\rm 5} School of Automation, Southeast University\\
    \textsuperscript{\rm 6} National University of Defense Technology\\
}

\begin{document}

\maketitle

\begin{abstract}
Class-incremental learning is commonly instantiated as a single-model paradigm, where a unified model sequentially adapts to an unbounded stream of sessions. While effective under mild distributional shifts, this formulation becomes strained when successive sessions induce incompatible optimization directions, leading to destructive interference and catastrophic forgetting. We argue that such forgetting reflects a structural limitation of enforcing heterogeneous learning dynamics within a single parameter space. Motivated by social solidarity theory, we propose Socialized Division and Collaboration (SDC) as a reformulation of continual learning that decomposes session learning across specialized models in response to optimization conflicts, while enabling coordinated collaboration. To support this formulation with a principled allocation mechanism, we introduce an energy-based session–model compatibility criterion grounded in Helmholtz free energy, which guides adaptive session allocation and model evolution under conflicting objectives. This framework integrates session assignment, model evolution, and collaborative inference into a unified pipeline, offering an alternative to monolithic continual learning formulations and highlighting a broader design principle for learning under persistent optimization conflicts.
\end{abstract}


\section{Introduction}

In machine societies, system level progress depends on continual learning and reliable knowledge integration. Classical class incremental learning (CIL) trains a single model sequentially over an unbounded stream of sessions~\cite{Zhou2023ClassIncrementalLA,zhang2025few}. This monolithic update paradigm, however, becomes fragile as sessions accumulate and their distributions diverge. Under substantial distribution shifts or conflicting objectives, CIL often fails to satisfy competing constraints simultaneously, resulting in incompatible parameter updates and catastrophic forgetting.
\begin{figure}[htbp!]  
  \centering
  \includegraphics[width=\columnwidth, keepaspectratio]{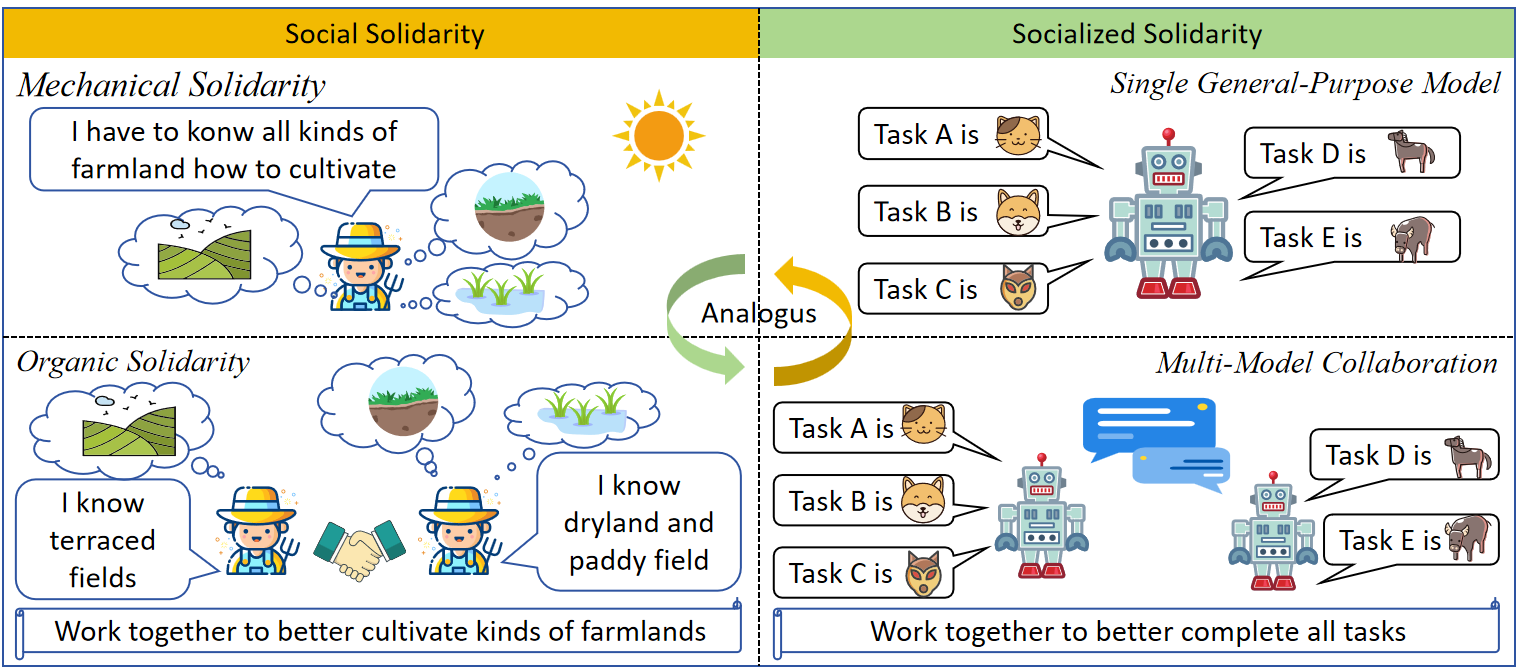}  
  \caption{Social solidarity in human society versus socialized solidarity in machine society.}  
  \label{fig:motivated}
\end{figure}
Human societal evolution highlights what is fundamentally missing. Early societies were organized around mechanical solidarity, where similar roles and shared knowledge made collective functioning straightforward. As societies expanded, they evolved toward organic solidarity, in which specialization and interdependence became essential~\cite{durkheim1997division}. The driving force of this transition was not increased individual capacity, but structured division of labor coupled with the systematic integration of collective intelligence~\cite{pmlr-v235-yao24d}.

In CIL, Low-Rank Adaptation (LoRA)~\cite{Hu2021LoRALA} partially reflects this principle by separating parameter updates across sessions. By constraining updates to low rank subspaces, LoRA reuses historical optimization directions to accommodate new data and mitigate forgetting. Nevertheless, this separation remains implicit and local. It neither resolves inter session interference when optimization directions conflict, nor provides a principled mechanism to align heterogeneous sessions with models of different learning capacities~\cite{23}. When sessions induce divergent gradients, distributional constraints prevent a single parameterization from retaining all knowledge, and existing variants fall short of genuine division and collaboration.

More broadly, current paradigms~\cite{yao2025socialized,li2025graphs} rarely achieve structural division. Most approaches perform separation at the data or parameter level, while the learning process remains dominated by session specific optimization directions and the models be systematically divided and collaboratively integrated~\cite{Zhang2023ExploringCM} so that the system can stably retain and refine prior knowledge while continuously assimilating new sessions? A careful examination of existing work reveals two unresolved issues. 

\begin{enumerate}
    \item[1:] \emph{How to dynamically assign sessions to models?}
    \item[2:] \emph{How to balance specialization and collaboration?}
\end{enumerate}

Guided by social solidarity theory, which assigns roles according to individual strengths, we propose the Socialized Division and Collaboration (SDC) paradigm. SDC reframes CIL as the evolution of a multi model societal system, in which sessions are explicitly allocated to models that can learn them with minimal interference, and knowledge is integrated through structured collaboration. By quantifying an affinity energy between sessions and models, SDC enables dynamic session allocation and division aware collaborative inference, thereby reducing forgetting and improving system level performance.

To instantiate SDC, we design the framework grounded in Helmholtz free energy. SDC consists of three core components: session division, model evolution, and collaborative inference. Sessions are assigned to models by minimizing Helmholtz free energy, which serves as a unified measure of session model compatibility. Models then evolve on their allocated sessions, and the system performs full classification through collaboration among specialized models, collectively alleviating catastrophic forgetting. These contributions are detailed as follows:
\begin{itemize}
\item We propose SDC as a practical paradigm for CIL, reframing continual learning through explicit division and structured collaboration.
\item We discuss that shared low-rank learning has an inherent bottleneck under heterogeneous sessions, which can be strictly reduced through division.
\item We realize SDC through the energy selection coevolution framework for dynamic grouping and division-aware collaboration.
\end{itemize}

\section{Related Work}

\subsection{Class-Incremental Learning}
\textbf{Class-Incremental Learning (CIL)} is a canonical continual learning setting in which a model must incrementally acquire new classes while retaining previously learned knowledge. To address catastrophic forgetting, a wide range of approaches have been proposed, which can be broadly grouped into three categories.
\textbf{(1) Rehearsal-based methods}~\cite{27,23,3,5} mitigate forgetting by storing or generating representative samples from earlier classes and replaying them alongside data from newly introduced classes.
\textbf{(2) Regularization-based methods}~\cite{26,6,7,8} discourage forgetting by constraining parameter updates for new classes, typically through importance-weighted penalties that limit changes to parameters critical for previously learned knowledge.
\textbf{(3) Parameter-isolation methods}~\cite{9,10,12,13} reduce interference by allocating class- or session-specific parameter subsets within a shared architecture, enforcing structural separation between representations learned at different stages.

Despite their algorithmic diversity, most existing CIL approaches operate under a single-model learning paradigm. In class-incremental settings, the optimization objectives induced by newly introduced classes often conflict with representations formed for earlier classes. When all classes are optimized within a shared parameter space, such conflicts are difficult to resolve and manifest as persistent interference between old and new knowledge. As learning proceeds over many increments, these unresolved conflicts accumulate, ultimately resulting in catastrophic forgetting.

\subsection{CIL with Foundation Models}
\textbf{CIL with foundation models} leverages large-scale pre-trained models to facilitate knowledge transfer across sessions and alleviate catastrophic forgetting in continual learning. By exploiting strong and transferable representations acquired during pre-training, these approaches aim to improve performance across successive incremental sessions. Existing methods can be broadly categorized into two groups.
\textbf{(1) Prompt-based adaptation methods}~\cite{Wang2021LearningTP,18,Smith2022CODAPromptCD,24} integrate Vision Transformers with prompt tuning, where session-specific prompts are selectively activated or composed to adapt a shared backbone, enabling continual adaptation while retaining previously learned knowledge.
\textbf{(2) LoRA-based methods}~\cite{Liang2024InfLoRAIL,23,25} support parameter-efficient continual learning by injecting session-specific information through low-rank adaptations while keeping the foundation backbone frozen, thereby reducing interference and training overhead.

Despite these advances, foundation model--based continual learning methods remain fundamentally constrained by the single-model paradigm. As new sessions are introduced, heterogeneous optimization objectives are forced to coexist within a shared parameter space, leading to accumulating knowledge conflicts that limit effective retention. This structural limitation is not unique to foundation models, but also characterizes traditional continual learning approaches. To overcome this bottleneck, we adopt a collaborative multi-model perspective with explicit division of labor, allowing different models to specialize in compatible sessions and thereby alleviating cross-session interference.

\section{Division-Aware CIL}
\label{sec:ch3}

This section characterizes cross-session inconsistency as the dominant source of error in low-rank adaptation and examines how free energy width governs this term at the level of session groups. These theoretical observations motivate a division-aware view of CIL and provide the foundation for the method developed in the next section. The full proofs of the theorems are provided in the appendix.

\textbf{Problem setup:} We consider a sequential stream of sessions $\{1,\ldots,N\}$, where each session $i$ is associated with a session-specific optimal update matrix $\Delta W_i^\star$. Our objective is to establish that grouping sessions according to a free energy score $F_i$ produces more coherent session groups with smaller within-group radii. As a result, the cross-session inconsistency term $\varepsilon_S$ is reduced, leading to a tighter upper bound on the approximation error.

\begin{definition}[Chebyshev center and radius]
\label{def:chebyshev}
Let $S \subseteq \{1,\ldots,N\}$ denote a set of session indices, where each session $i \in S$ is associated with a
session-specific optimal update matrix $\Delta W_i^\star$. Under the operator norm $\|\cdot\|_{\op}$, the
Chebyshev center of the set $\{\Delta W_i^\star\}_{i\in S}$ is defined as
\begin{equation}
\Delta W_S^\star \in \arg\min_{X}\ \max_{i\in S}\ \|\Delta W_i^\star - X\|_{\op},
\end{equation}
and the corresponding Chebyshev radius is
\begin{equation}
\varepsilon_S := \min_{X}\ \max_{i\in S}\ \|\Delta W_i^\star - X\|_{\op}.
\end{equation}
\end{definition}

The Chebyshev radius $\varepsilon_S$ measures the maximal dispersion of session-specific optimal updates within the set $S$
around an optimally chosen reference point. A larger $\varepsilon_S$ indicates weaker within-set alignment and
greater difficulty for parameter sharing, whereas a smaller $\varepsilon_S$ corresponds to stronger alignment and
more favorable sharing conditions.

\begin{theorem}[Low-rank approximation with heterogeneity floor]
\label{thm:low_rank_approx}
Building on prior work~\cite{23}, suppose Assumption~\ref{assump:dispersion} holds.
Let $\Delta W^{[k]}$ denote the rank-$k$ approximation of $\Delta W^\star$ obtained by retaining the top-$k$
principal components. Under standard spectral separation and optimization conditions,
there exist numerical constants $c$ and $c'$ such that, with high probability, gradient descent produces iterates
satisfying
\begin{equation}
\|A_{i_k} B_{i_k} - \Delta W^{[k]}\|_{\op}
\le \epsilon \sigma_1 + \varepsilon_S, \quad \forall k = 1,2,\ldots,j.
\end{equation}
\end{theorem}

\begin{remark}[Irreducible heterogeneity floor]
The bound in Theorem~\ref{thm:low_rank_approx} decomposes the approximation error into an optimization term and an
irreducible heterogeneity term $\varepsilon_S$. While the optimization term can be reduced through training,
$\varepsilon_S$ depends solely on cross-session inconsistency within the sharing set and therefore constitutes a
fundamental bottleneck for shared low-rank adaptation.
\end{remark}

\begin{assumption}[Free energy Lipschitz consistency]
\label{ass:lipschitz_hp}
There exists a constant $L_F>0$ such that for any unordered session pair $\{p,q\}$
(with $p\neq q$),
\begin{equation}
\Pr\!\Big(
\|\Delta W_p^\star-\Delta W_q^\star\|_{\op}
\le L_F\,|F_p-F_q|+\eta_\delta
\Big)
\ \ge\ 1-\delta.
\label{eq:lipschitz_hp_pair}
\end{equation}
where $\delta\in(0,1)$ is a confidence level and $\eta_\delta\ge 0$ is a residual term.
\end{assumption}

\begin{definition}[Free energy grouping]
\label{def:free_energy_grouping}
Let the session indices $\{1,\ldots,N\}$ be partitioned into $L$ disjoint groups
$\{G_\ell\}_{\ell=1}^{L}$ according to their free energy scores $\{F_i\}$.
The partition is called a free energy grouping with width $\tau$
if, for every group $G_\ell$, all session pairs $(p,q)\in G_\ell$ satisfy:
\end{definition}

\begin{lemma}[Group-level error control under free energy grouping]
\label{lem:group_hp}
Assume Assumption~\ref{ass:lipschitz_hp} and Assumption~\ref{assump:sdlora} hold.
Under the free energy width condition in Definition~\ref{def:free_energy_grouping},
for any group $G_\ell$ define $\delta_\ell := \binom{|G_\ell|}{2}\delta$.
Then there exists an event $\mathcal{E}_{G_\ell}$ with
$\Pr(\mathcal{E}_{G_\ell}) \ge 1 - \delta_\ell$
such that, on $\mathcal{E}_{G_\ell}$,
\begin{equation}
\resizebox{0.9\linewidth}{!}{$ 
\begin{aligned}
    \max_{i\in G_\ell}
    \|A_iB_i - \Delta W_{G_\ell}^{[:j_\ell]}\|_{\op}
    &\le \epsilon\,\sigma_1(\Delta W_{G_\ell}^{\star}) + \varepsilon_{G_\ell} \\
    &\le \epsilon\,\sigma_1(\Delta W_{G_\ell}^{\star}) + (L_F\tau + \eta_\delta).
\end{aligned}
$}
\label{lemma16}
\end{equation}
In contrast, without grouping, the global worst-case session satisfies
\begin{equation}
\max_{1\le i \le N}
\|A_iB_i - \Delta W_{1:N}^{[:j]}\|_{\op}
\le \epsilon\,\sigma_1(\Delta W_{1:N}^{\star}) + \varepsilon_{1:N}.
\label{eq:global_worst}
\end{equation}
\end{lemma}
\begin{remark}[Effect of grouping on heterogeneity]
Lemma~\ref{lem:group_hp} shows that free energy grouping replaces the global inconsistency term
$\varepsilon_{1:N}$ with smaller within-group heterogeneity terms $\varepsilon_{G_\ell}$,
which is the key mechanism through which grouping improves the worst-case error bound.
\end{remark}

 \begin{figure*}[ht!]  
  \centering
  \includegraphics[
    width=\textwidth,        
    height=0.4\textheight,   
    keepaspectratio          
  ]{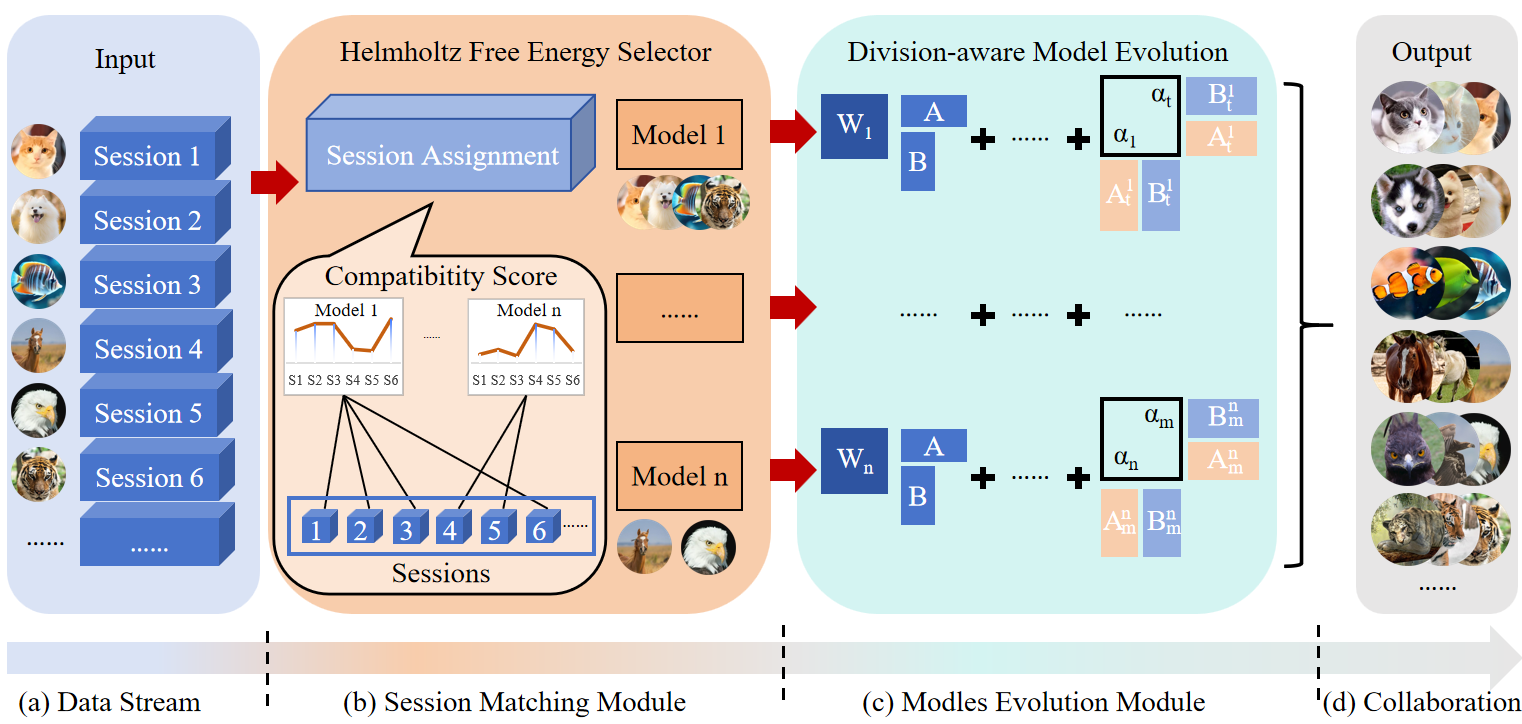}  
  \caption{Overview of the proposed Socialized Division and Collaboration (SDC) framework. SDC consists of a Helmholtz free energy selector for session division, a model evolution module for specialized learning, and a collaboration module for structured integration, jointly enabling division-aware and collaboration-aware continual learning.}  
  \label{fig:sdc}
\end{figure*}

\begin{theorem}[Strict improvement via free energy grouping]
\label{mainthm}
Under the conditions of Lemma~\ref{lem:group_hp}, suppose $0 < \epsilon < \tfrac{1}{2}$ and
the grouping width $\tau$ satisfies
\begin{equation}
L_F\tau + \eta_\delta < (1-2\epsilon)\,\varepsilon_{1:N}.
\label{thm8}
\end{equation}
Then the grouped bound is strictly tighter than the global sharing bound, i.e.,
\begin{equation}
\max_\ell
\bigl(
\epsilon\,\sigma_1(\Delta W_{G_\ell}^{\star}) + \varepsilon_{G_\ell}
\bigr)
<
\epsilon\,\sigma_1(\Delta W_{1:N}^{\star}) + \varepsilon_{1:N}.
\label{tighter_bound}
\end{equation}
\end{theorem}

\begin{remark}[Implication for division-aware learning]
Theorem~\ref{mainthm} formally establishes that free energy grouping can strictly improve the worst-case
approximation guarantee by controlling cross-session inconsistency. When the grouping width is sufficiently
small, structured division replaces global sharing with localized consistency, yielding provably sharper
error bounds.
\end{remark}

\section{Methodology}
\noindent To address the optimization conflicts inherent in class-incremental learning, we propose the Socialized Division and Collaboration (SDC) framework, illustrated in Figure~\ref{fig:sdc}. Rather than forcing all incremental sessions to be absorbed by a single model, SDC reformulates continual learning as a structured system with explicit division of labor and controlled collaboration.

SDC is instantiated through a three-stage design. First, a Helmholtz free energy selector performs principled session division by grouping incoming sessions according to their energetic compatibility, thereby preventing severe cross-session interference at the source. Second, in the decoupled evolution stage, specialized and parameter-efficient modules adapt to their assigned session groups, allowing knowledge to evolve without corrupting unrelated representations. Finally, a collaboration-aware training objective enables structured integration among specialized modules, ensuring that complementary expertise can be jointly leveraged at inference time. Together, these components form a coherent mechanism that mitigates catastrophic forgetting while supporting continual knowledge accumulation.

\noindent\textbf{Helmholtz free energy selector (HFES):}
To resolve optimization conflicts at their origin, we adopt an energy-based assignment mechanism that allows each incoming session to be claimed by the model with which it is most compatible. The core idea is to quantify the alignment between sessions and models through Helmholtz free energy and to derive a principled compatibility score for model selection. For an input $x$ and model $k$, we define the energy of class $y$ as
\begin{equation}
E_k(x,y) = -h_k(x)[y],
\end{equation}
where $h_k(x)$ denotes the logits produced by model $k$. The corresponding Helmholtz free energy is defined as
\begin{equation}
F_k(x) = -\log \sum_{y \in \mathcal{Y}_k} \exp\!\big(-E_k(x,y)\big).
\end{equation}
Since $F_k(x)$ is negative and lower values indicate stronger overall evidence under the induced Gibbs distribution, we introduce a compatibility score defined by
\begin{equation}
S_k(x) \triangleq -F_k(x),
\end{equation}
where larger values correspond to stronger alignment between the session and the model.

Given a session dataset $\mathcal{D}_i$, we evaluate its compatibility with each model by aggregating sample-level scores,
\begin{equation}
S_i^{(k)} = \mathbb{E}_{x \sim \mathcal{D}_i}\!\left[S_k(x)\right],
\end{equation}
and assign the session to the model that maximizes scores,
\begin{equation}
j_i^\star = \arg\max_k S_i^{(k)}.
\end{equation}
The selected model then exclusively updates its parameters on $\mathcal{D}_i$, thereby establishing ownership of the session. As learning proceeds, this energy-based assignment induces stable specialization across models and enforces an implicit division of labor.

\noindent\textbf{Division-aware model evolution (DME):}
After session assignment by the Helmholtz free energy selector, each session is routed to a dedicated expert model and contributes only to the evolution of that expert. To ensure parameter efficiency and prevent catastrophic forgetting, the shared backbone $W_0$ is kept frozen, and adaptation is realized through sequential low-rank updates. For a given expert, model evolution proceeds over sessions $\{1,\dots,N\}$ as
\begin{equation}
W^{(n)} = W_0 + \sum_{i=1}^{n} \Delta W_i ,
\end{equation}
where each update $\Delta W_n = A_n B_n$ is a rank-$r$ matrix learned exclusively from the data of session $n$. Importantly, each $\Delta W_n$ is constrained to capture novel update directions relative to previously learned ones, ensuring that successive low-rank updates complement rather than overwrite earlier knowledge.

This formulation can be viewed as a structured, expert-local sequential low-rank approximation of session-specific update directions. By restricting updates to aligned subspaces within each expert and avoiding cross-expert interference, the model incrementally captures dominant variations induced by new sessions while preserving previously acquired representations. Crucially, this evolution mechanism builds on the stable session grouping induced by free energy-based assignment, ensuring that low-rank updates remain consistent with the expert’s optimization trajectory.

\noindent\textbf{Division-based collaboration (DC):}
The final stage enables system-level collaboration by jointly leveraging the evolved and decoupled expert models. While the preceding stages ensure that sessions are compatibly assigned and experts are independently optimized, this stage governs how specialized expertise is integrated without violating the established division of labor.

\begin{table*}[t]
  \centering
  \small 
  \setlength{\tabcolsep}{2pt} 
  \begin{tabular}{l c@{\,}l c@{\,}l c@{\,}l c@{\,}l c@{\,}l c@{\,}l}
    \toprule
    \multirow{2}{*}{\textbf{Method}} & \multicolumn{4}{c}{\textbf{ImageNet-R}} & \multicolumn{4}{c}{\textbf{CIFAR-100}} & \multicolumn{4}{c}{\textbf{CUB-200}} \\
    \cmidrule(lr){2-5} \cmidrule(lr){6-9} \cmidrule(lr){10-13}
    & ACC (↑) & & AAA (↑) & & ACC (↑) & & AAA (↑) & & ACC (↑) & & AAA (↑) & \\
    \midrule
    Fine-Tuning\cite{Shuttleworth2024LoRAVF} & 60.57 & & 72.31 & & 69.49 & & 80.35 & & 51.43 & & 69.74 & \\
    L2P~\cite{Wang2021LearningTP}            & 71.26 & & 76.13 & & 83.18 & & 87.69 & & 65.18 & & 76.12 & \\
    DualPrompt~\cite{18}                     & 68.22 & & 73.81 & & 81.48 & & 86.41 & & 68.00 & & 79.40 & \\
    CODA-Prompt~\cite{Smith2022CODAPromptCD} & 74.05 & & 78.14 & & 86.31 & & 90.67 & & 71.92 & & 78.76 & \\
    InfLoRA~\cite{Liang2024InfLoRAIL}        & 74.75 & & 80.67 & & 86.75 & & 91.72 & & 70.82 & & 81.39 & \\
    SD-LoRA~\cite{23}                        & \underline{77.34} & & \underline{82.04} & & \underline{88.01} & & \underline{92.54} & & \underline{77.48} & & \underline{85.59} & \\
    \midrule
    \textbf{SDC(Ours)}                       & \textbf{85.96} & \textbf{$_{(+8.62)}$} & \textbf{88.61} & \textbf{$_{(+6.57)}$} & \textbf{95.08} & \textbf{$_{(+7.07)}$} & \textbf{96.58} & \textbf{$_{(+4.04)}$} & \textbf{84.93} & \textbf{$_{(+7.45)}$} & \textbf{90.88} & \textbf{$_{(+5.29)}$} \\
    \bottomrule
  \end{tabular}
  \caption{Performance of different methods on ImageNet-R, CIFAR-100 and CUB-200. The 1$^\mathrm{st}$ and 2$^\mathrm{nd}$ best results are highlighted in \textbf{bold} and \underline{underlined}, respectively.}
  \label{tab:merged_perf_all_datasets}
\end{table*}

Crucially, collaboration is induced during training through a division-preserving objective defined at the system level. Rather than jointly Fine-Tuning all sessions within a single model, we optimize a composite objective that aggregates the losses of all experts:
\begin{equation}
\mathcal{L}_{\text{total}} = \sum_{k=1}^{K} \mathcal{L}_k ,
\end{equation}
where $\mathcal{L}_k$ denotes the empirical loss over sessions assigned to expert $k$. This formulation couples experts only at the objective level, while parameter updates remain strictly localized within each model.

As a result, collaboration emerges naturally at inference time: multiple experts can be jointly activated to produce predictions without parameter merging or post-hoc coordination, yielding a coherent system that balances specialization and collaboration.

\noindent In summary, SDC operates as a unified three-stage pipeline encompassing division, evolution, and collaboration. For clarity, we provide explicit training and inference procedures in Algorithm~\ref{alg:sdc_training} and Algorithm~\ref{alg:sdc_inference}.

\begin{algorithm}[H]
  \caption{Training for SDC.}
  \label{alg:sdc_training}
  \begin{algorithmic}[1]
    \STATE {\bfseries Input:} backbone $W_0$, session dataset $\{\mathcal{D}_i\}_{i=1}^N$, compatibility score function $S(\cdot;\cdot)$, LoRA rank $r$
    \STATE {\bfseries Output:} expert adapters $\{(A_k,B_k)\}_{k=1}^K$
    \STATE Initialize $K$ adapters $\{(A_k,B_k)\}_{k=1}^K$ with rank $r$
    \STATE Initialize empty session ownership sets $\{\mathcal{S}_k\}_{k=1}^K$
    \FOR{$i = 1$ \textbf{to} $N$}
      \STATE Compute scores $\{S_i^{(k)}\}_{k=1}^K$ for session datasets $\mathcal{D}_i$
      \STATE Select expert index $k_i^\star \gets \arg\max_{k \in \{1,\dots,K\}} S_i^{(k)}$
      \STATE Assign session datasets $\mathcal{D}_i$ to expert $k_i^\star$
      \STATE Update adapter $(A_{k_i^\star}, B_{k_i^\star})$ on session dataset $\mathcal{D}_i$
      \STATE Add $\mathcal{D}_i$ to ownership set $\mathcal{S}_{k_i^\star}$
    \ENDFOR
  \end{algorithmic}
\end{algorithm}

\begin{algorithm}[H]
  \caption{Inference for SDC.}
  \label{alg:sdc_inference}
  \begin{algorithmic}[1]
    \STATE {\bfseries Input:} test sample $x$, backbone $W_0$, expert adapters $\{(A_k,B_k)\}_{k=1}^K$, compatibility score function $S(\cdot;\cdot)$
    \STATE {\bfseries Output:} prediction $\hat y$
    \STATE Compute compatibility scores $\{S_k(x)\}_{k=1}^K$
    \STATE Select expert index $k^\star \gets \arg\max_{k \in \{1,\dots,K\}} S_k(x)$
    \STATE Compute prediction $\hat y \gets f(x; W_0 + A_{k^\star}B_{k^\star})$
  \end{algorithmic}
\end{algorithm}

\section{Experiments}
We evaluate Socialized Division and Collaboration (SDC) on three widely adopted benchmarks: ImageNet-R~\cite{Boschini2022TransferWF}, CIFAR-100~\cite{Krizhevsky2009LearningML}, and CUB-200~\cite{2011The}. Beyond reporting overall performance, our experiments are designed to examine how explicit division of labor and structured collaboration influence knowledge evolution and forgetting in CIL. All experiments are implemented in PyTorch and conducted on two NVIDIA RTX 4090 GPUs.

\subsection{Implementation Details}

We compare SDC with several state-of-the-art (SOTA) methods, including Fine-Tuning, L2P~\cite{Wang2021LearningTP}, DualPrompt~\cite{18}, CODA-Prompt~\cite{Smith2022CODAPromptCD}, InfLoRA~\cite{Liang2024InfLoRAIL} and SD-LoRA~\cite{23}. All datasets are divided into 10 incremental sessions, ACC measures accuracy, AAA measures forgetting rate, and the formal definitions of the evaluation metrics can be found in the appendix.More experiments and details of implementation refer to appendix.

\subsection{Division and Collaboration Are All You Need}

As shown in Table~\ref{tab:merged_perf_all_datasets}, SDC achieves the best performance across all three benchmarks. Notably, these improvements hold across diverse domains and class distributions, indicating that the gains are systematic rather than dataset-specific.

\textbf{Neither division nor collaboration alone is sufficient:}
A closer analysis exposes fundamental structural limitations in existing approaches. Single-model methods such as SD-LoRA mitigate catastrophic forgetting by decoupling sessions, yet all updates remain confined within a shared parameter space. As session heterogeneity increases, this constraint inevitably induces representational interference and leads to diminishing returns. Conversely, traditional ensemble methods benefit from collaboration across models, but the lack of explicit session assignment leaves optimization conflicts unresolved during model evolution, which undermines both scalability and stability.

\textbf{Division-aware collaboration resolves both limitations:}
By assigning each session to the most compatible model via a free energy-based criterion, SDC eliminates optimization conflicts at their source. This principled division of labor allows models to specialize without interference, while structured collaboration enables complementary knowledge to be effectively integrated. As a result, SDC delivers substantial and robust performance improvements together with strong suppression of forgetting.

\subsection{Specialization-Complementarity Trade-Off}

As shown in Table~\ref{tab:merged_model_num_singlecol1} and Figure~\ref{fig:method_performance}, our results reveal a clear trade-off between model specialization and inter-session complementarity. On ImageNet-R and CUB-200, increasing the number of models beyond a dataset-dependent threshold leads to higher forgetting rates.
To fundamentally investigate this threshold, we conducted a data-driven semantic clustering analysis on the datasets. The results demonstrate that the underlying feature spaces naturally partition into three highly cohesive macro-clusters, effectively proving that K=3 is the optimal capacity threshold to balance specialization and complementarity.More experiments and details of the semantic clustering analysis refer to appendix~\ref{appendix:k_selection}.
\begin{table}[!h]
  \centering
  \small  
  \setlength{\tabcolsep}{2pt} 
  \begin{tabular}{l c c c c c c}
    \toprule
    \multirow{2}{*}{\textbf{Method}} & \multicolumn{2}{c}{\textbf{ImageNet-R}} & \multicolumn{2}{c}{\textbf{CIFAR-100}} & \multicolumn{2}{c}{\textbf{CUB-200}} \\
    \cmidrule(lr){2-3} \cmidrule(lr){4-5} \cmidrule(lr){6-7}
    & ACC ($\uparrow$) & AAA ($\uparrow$) & ACC ($\uparrow$) & AAA ($\uparrow$) & ACC ($\uparrow$) & AAA ($\uparrow$) \\
    \midrule
    SDC-M2 & 83.99 & 87.02 & 93.36 & 95.87 & 83.94 & \underline{90.26} \\
    SDC-M3 & \textbf{85.96} & \textbf{88.61} & \underline{95.08} & \underline{96.58} & \textbf{84.93} & \textbf{90.88} \\
    SDC-M4 & \underline{85.85} & \underline{88.29} & \textbf{96.13} & \textbf{97.40} & \underline{84.80} & 89.50 \\
    \bottomrule
  \end{tabular}
  \caption{Effect of the number of models on SDC performance. The 1$^\mathrm{st}$ and 2$^\mathrm{nd}$ best results are highlighted in \textbf{bold} and \underline{underlined}, respectively.}
  \label{tab:merged_model_num_singlecol1}
\end{table}

\begin{figure}[htbp!]  
  \centering
  \includegraphics[width=\columnwidth, keepaspectratio]{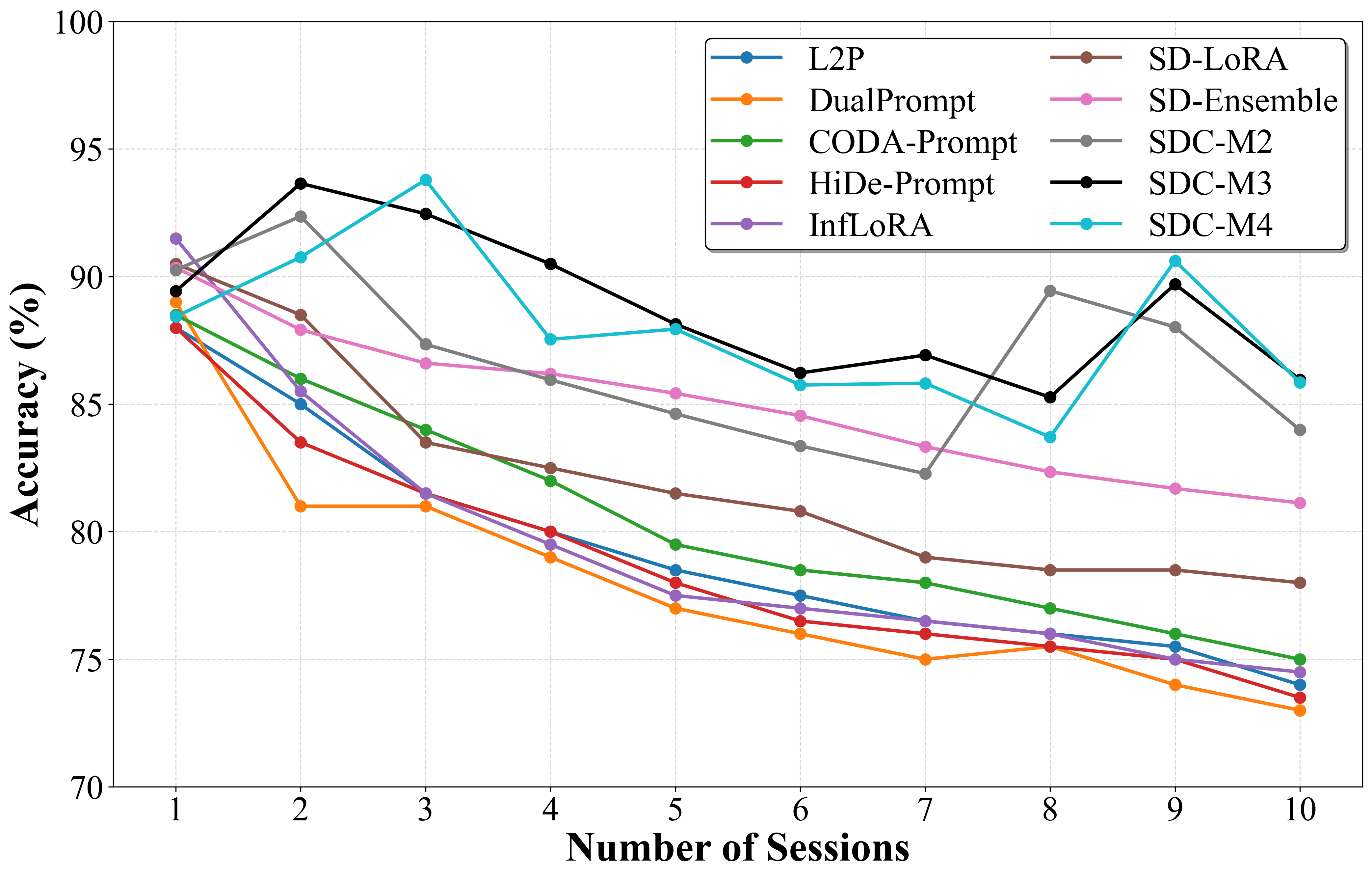}  
  \caption{Detailed accuracy of each session on ImageNet-R.}  
  \label{fig:method_performance}
\end{figure}

\textbf{Session-driven division of labor:} Crucially, in CIL scenarios, data arrives sequentially, requiring the learning system to adapt to the evolving task distribution. Accordingly, the model's capacity ($K$) within our framework is fundamentally task-driven. To theoretically validate this environment-dependent nature, we conducted a retrospective semantic correlation analysis on the complete datasets. This analysis reveals that the underlying feature spaces of the three evaluated benchmarks naturally partition into exactly three highly cohesive macro-clusters. Therefore, the optimal performance observed at $K=3$ is a direct reflection of the specific semantic boundaries inherent to these datasets. Extending this principle to broader real-world applications, the value of $K$ must explicitly correspond to the specific demands of the task scenario. Highly heterogeneous environments require a larger $K$ to isolate disjoint tasks and ensure strict specialization, whereas more homogeneous streams favor a smaller $K$ to prevent the fragmentation of shared knowledge. Ultimately, SDC establishes a flexible division-and-collaboration mechanism, demonstrating that robust continual learning naturally aligns the system's structural capacity with the inherent complexity and semantic diversity of the session stream.

\subsection{Analysis of Division Strategy Effects}

To elucidate the underlying mechanism of SDC, we compare three session division strategies: (i) uniform random division, (ii) imbalanced division (an extreme 1:9 split), and (iii) our free energy-guided division. Comprehensive evaluations across all datasets are detailed in appendix~\ref{ablation study}. The results reveal a consistent pattern: imbalanced and random divisions suffer from pronounced performance degradation due to broken complementarity and unaligned model capacities. In contrast, our division-aware collaboration achieves the strongest performance by assigning each session to the most compatible model.

\textbf{Free energy-guided assignment is essential for division:} The results reveal a clear and consistent pattern. Imbalanced division leads to pronounced performance degradation, indicating that specialization without sufficient complementarity exacerbates catastrophic forgetting. Uniform random division yields only marginal improvements and remains substantially inferior to division-aware collaboration, as it fails to align sessions with appropriate model capacity. In contrast, free energy-guided assignment achieves the strongest performance by allocating each session to the model whose optimization characteristics are most compatible.

\subsection{Ablation Study}
We conduct a series of ablation experiments to disentangle and quantify the individual contributions of the division-collaboration mechanism and the model evolution component in SDC. The results, summarized in Table~\ref{tab:ablation_exp}, provide clear evidence for the necessity of each module.

\begin{table}[htbp]
  \centering
  \setlength{\tabcolsep}{1mm}   
  \small                         
  \begin{tabular}{ccc c@{}l c@{}l}
    \toprule
    \textbf{Method} & \textbf{HFES} & \textbf{DME} & ACC ($\uparrow$) & & AAA ($\uparrow$) & \\
    \midrule
    \multirow{4}{*}{\textbf{SDC}} 
          & $\times$     & $\times$       & 70.62 & & 76.32 & \\
          & $\times$     & $\checkmark$   & 79.86 & & 84.62 & \\
          & $\checkmark$ & $\times$       & \underline{82.15} & & \underline{86.15} & \\
          & $\checkmark$ & $\checkmark$   & \textbf{85.96} & \textbf{$_{(+3.81)}$} & \textbf{88.61} & \textbf{$_{(+2.46)}$} \\
    \bottomrule
  \end{tabular}
  \caption{Ablation study on ImageNet-R. The 1$^\mathrm{st}$ and 2$^\mathrm{nd}$ best results are highlighted in \textbf{bold} and \underline{underlined}, respectively.}
  \label{tab:ablation_exp}
\end{table}

Removing both HFES and DME causes severe degradation, as sessions compete within an undifferentiated parameter space, exacerbating knowledge conflicts. Using only the DME module yields moderate improvements by absorbing new knowledge, but unresolved cross-session optimization conflicts constrain further gains. Conversely, employing only the HFES module substantially improves performance via effective session-to-model matching, yet experts remain suboptimal without localized evolution. Finally, combining both modules delivers the strongest results: HFES aligns sessions with compatible models, while DME ensures full adaptation. Together, they synergistically underpin the superior and stable performance of SDC.

\subsection{Key Factors of CIL}

\begin{figure}[t!]
    \centering
    \begin{subfigure}[b]{0.95\columnwidth}
        \centering
        \includegraphics[width=\linewidth]{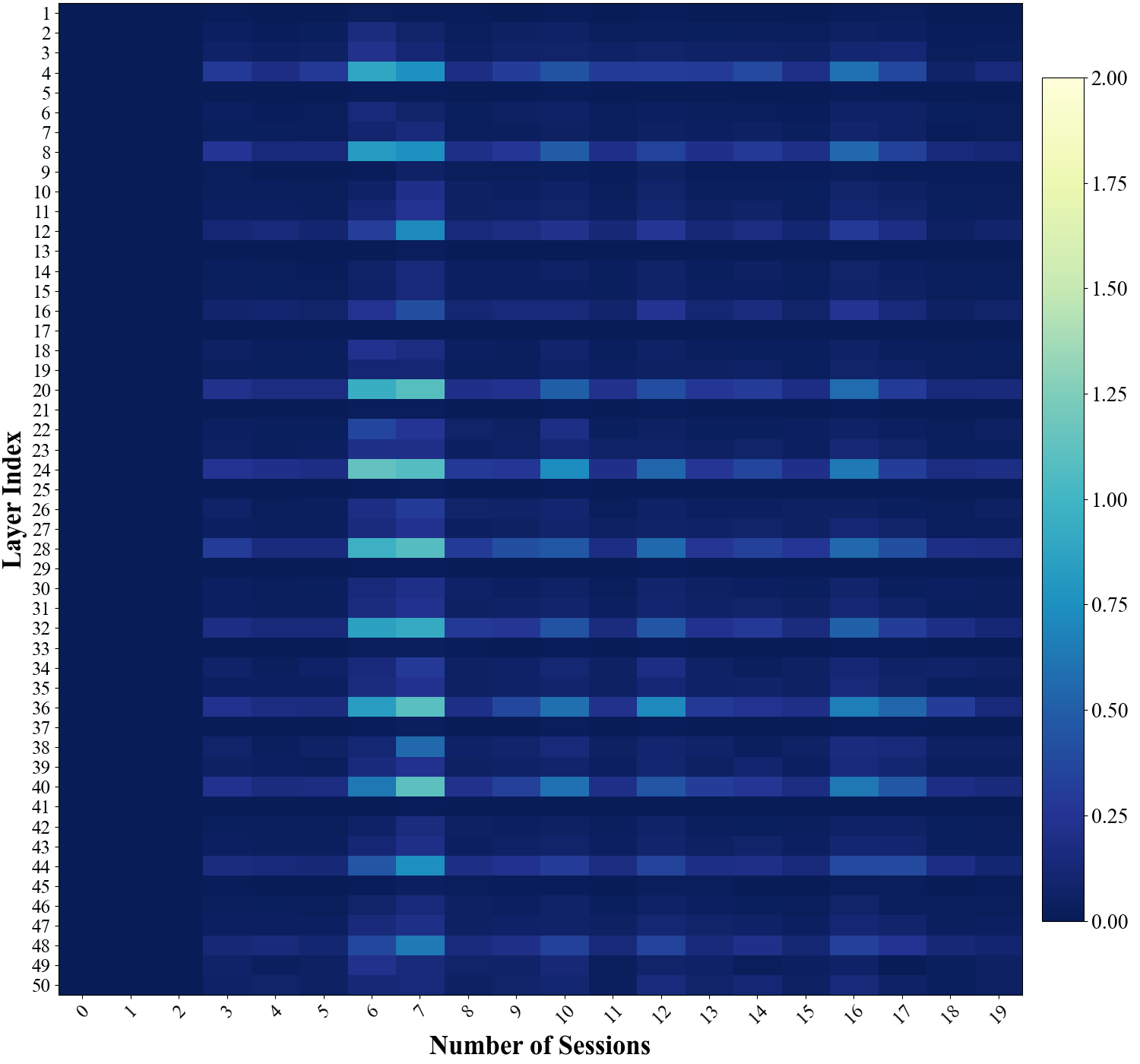}
        \caption{Gradient update of the SDC on the ImageNet-R}
        \label{subfig:left}
    \end{subfigure}
    
    \begin{subfigure}[b]{0.95\columnwidth}
        \centering
        \includegraphics[width=\linewidth]{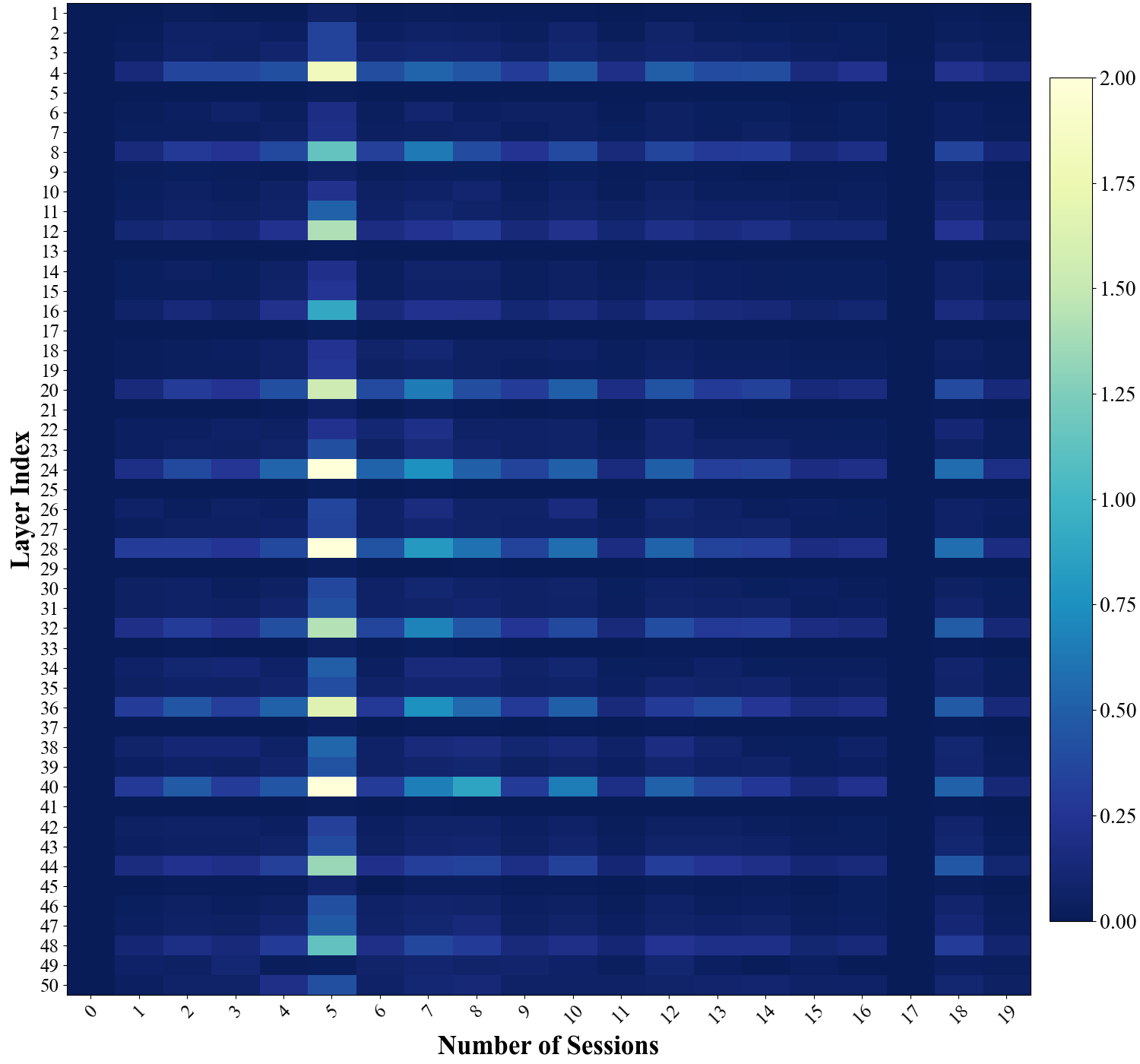}
        \caption{Gradient update of the SD-LoRA on the ImageNet-R}
        \label{subfig:right}
    \end{subfigure}
    
    \caption{Comparing the gradient updates of different methods.}
    \label{fig:two_subfigs}
\end{figure}

To directly examine how optimization conflicts manifest at the parameter-update level, we conduct a comparative analysis on ImageNet-R. Figure~\ref{fig:two_subfigs} visualizes gradient magnitude heatmaps across incremental sessions for SDC (Figure~\ref{subfig:left}) and SD-LoRA (Figure~\ref{subfig:right}), where rows correspond to network layers and columns to incremental sessions.

\textbf{Optimization conflict as a key factor of instability in CIL:}
Figure~\ref{subfig:right} shows that SD-LoRA suffers from sharp, high-magnitude gradient spikes, indicating abrupt parameter corrections driven by cross-session interference and incompatible objectives—a direct precursor to forgetting. Conversely, SDC (Figure~\ref{subfig:left}) exhibits a smooth, uniformly distributed gradient evolution without localized spikes. This confirms that SDC successfully resolves optimization conflicts at their source, enabling stable updates and robust knowledge retention.

\begin{table}[!t]
  \centering
  \small  
  \setlength{\tabcolsep}{1.5pt} 
  \begin{tabular}{@{} l c c c @{}}
    \toprule
           &        & Learnable  & Stored   \\
    Method & GFLOPs & Parameters & Features \\
           &        & (M)        & (M)      \\
    \midrule
    L2P\cite{Wang2021LearningTP}             & 70.14  & 0.48 & 0    \\
    DualPrompt\cite{18}                      & 70.26  & 0.06 & 0    \\
    CODA-Prompt\cite{Smith2022CODAPromptCD}  & 70.61  & 0.38 & 0    \\
    HiDe-Prompt\cite{16}                     & 70.36  & 0.08 & 0.10 \\
    InfLoRA\cite{Liang2024InfLoRAIL}         & 35.12  & 0.37 & 0.10 \\
    SD-LoRA\cite{23}                         & 35.12  & 0.37 & 0    \\
    SDC (Ours)                               & 16.94  & 0.37 & 0    \\
    \bottomrule
  \end{tabular}
  \caption{Comparison of computational cost, learnable parameters, and memory overhead across different methods.}
  \label{tab:model_complexity}
\end{table}

\subsection{Computational Efficiency and Resource Analysis}

We further evaluate the computational cost and resource footprint of different methods, with a detailed comparison reported in Table~\ref{tab:model_complexity}.

\textbf{SDC achieves a favorable accuracy-efficiency balance:}
As demonstrated in Table~\ref{tab:model_complexity}, SDC achieves strong performance with significantly lower computational and storage costs compared to existing methods. Unlike prior continual learning approaches that often sacrifice accuracy for efficiency, SDC avoids this trade-off. By merging division-aware session assignment with lightweight, parameter-efficient model evolution, SDC drastically cuts redundant computation and prevents unnecessary parameter growth. This equilibrium makes SDC highly suitable for resource-constrained, real-time scenarios where both predictive accuracy and operational efficiency are vital. The findings confirm that SDC boosts learning stability and performance while operating with a much smaller computational footprint.
\section{Conclusion}

This work introduces SDC as a practical paradigm for CIL, reframing continual learning from monolithic model updates to a structured system with explicit division of labor and controlled collaboration. From a theoretical perspective, we show that shared low-rank adaptation under heterogeneous sessions is fundamentally constrained by an irreducible inconsistency bottleneck. Crucially, this limitation does not stem from insufficient model capacity, but from conflicting optimization directions across sessions, and can be strictly alleviated through principled division. Guided by this insight, we instantiate SDC with an energy selection coevolution framework, which dynamically assigns sessions to compatible models and enables division-aware collaboration during learning and inference.Our analysis establishes that the model's capacity ($K$) must be dynamically scaled based on the semantic diversity of the session scenario. This design achieves strong empirical performance while preserving parameter efficiency and substantially mitigating forgetting.


\clearpage
\section{Appendix}
The appendix contains comprehensive details on the implementations
and experimental results referenced in the main paper, along with supplementary
theoretical analysis and in-depth discussions. It is organized as follows:
\begin{itemize}
    \item In Implementation Details, we offer a thorough overview of the methods compared in the main paper, accompanied by a detailed description of the datasets utilized.
    \item In Theoretical Analysis, we present the full set of experimental results, accompanied by an in-depth analysis that thoroughly evaluates the model’s performance.
    
    \item In Full Experimental Results, we provide detailed theoretical
proofs for the main results, including assumptions, lemmas, and theorems that formally justify the free energy grouping mechanism and its tighter optimization bounds.
\end{itemize}

\section{Implementation Details}
\label{sec:impl_details}

In this section, we offer a detailed description of the methods compared in the main paper, as well as the datasets used.
For classification as the primary session,the model is trained on 30 epochs using the AdamW optimizer with a batch size of 128. The learning rate is set to $8 \times 10^{-3}$ with a cosine learning rate scheduler and 4 warmup epochs. We applied weight decay of $5 \times 10^{-4}$ and gradient clipping of 1.0. The training process incorporates various data augmentation and regularization techniques, including mixup with a mix scale of 0.2, cutmix with a scale of 0.3, and random erase with a probability of 0.1.

\subsection{Compared Methods}

In this subsection, we provide an overview of the methods compared in the main paper, outlining their key characteristics.
The methods considered are as follows:
\begin{itemize}
    \item \textbf{Fine-Tuning \cite{Shuttleworth2024LoRAVF}:}
    Reveals that LoRA introduces high-rank “intruder dimensions” absent in full Fine-Tuning, showing these dimensions cause localized forgetting and accumulate harmfully in continual learning.
    
    \item \textbf{L2P \cite{Wang2021LearningTP}:}
    introduces dynamic prompt learning to guide pretrained models without session identity or rehearsal, enabling effective session-invariant and session-specific knowledge management for rehearsal-free continual learning.

    \item \textbf{DualPrompt \cite{18}:}
    Learns complementary session-invariant and session-specific prompts to guide pretrained models, achieving strong rehearsal-free continual learning without storing past data.

    \item \textbf{CODA-Prompt \cite{Smith2022CODAPromptCD}:}
    Introduces decomposed attention-based prompting that assembles input-conditioned prompts to enhance plasticity and reduce forgetting in rehearsal-free continual learning.
    
    \item \textbf{HiDe-Prompt \cite{16}:}
    Decomposes prompt-based continual learning into hierarchical components and optimizes them jointly, enabling robust session-specific adaptation and superior performance under self-supervised pretraining.

\begin{figure*}[t] 
    \centering
    \includegraphics[width=0.95\textwidth]{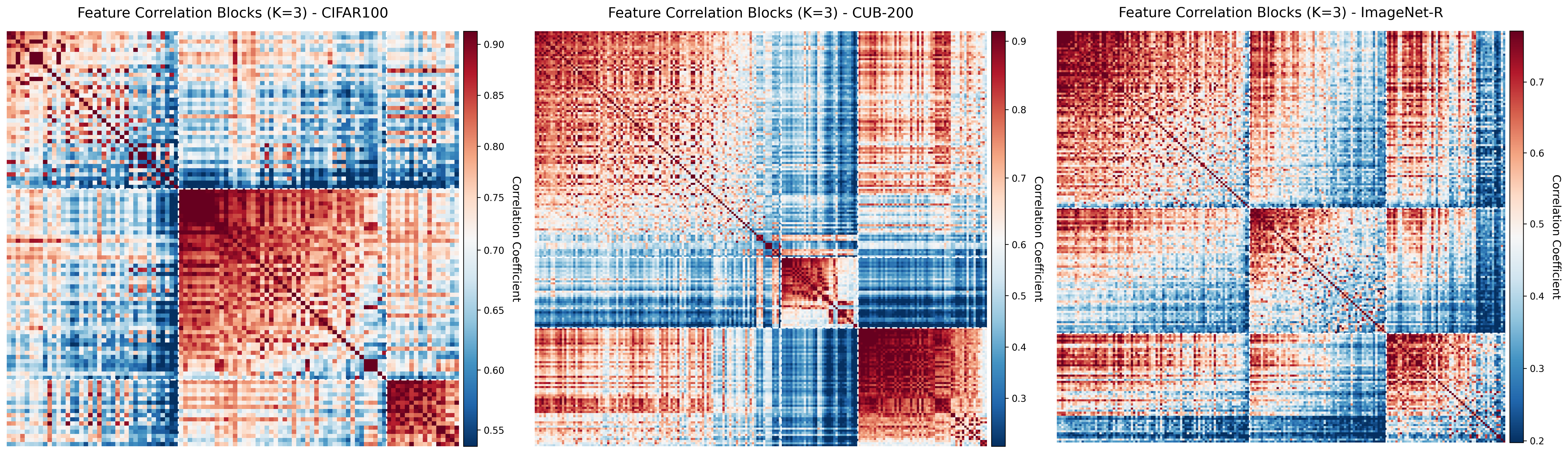} 
    \caption{Data-driven semantic correlation heatmaps of class centroids across CIFAR-100, CUB-200, and ImageNet-R. The distinct block-diagonal patterns (red blocks) naturally emerge under Ward's clustering and Two-Slope normalization, demonstrating that the underlying data inherently partitions into three major semantic groups. This inherent distribution strictly dictates our choice of $K=3$.}
    \label{fig:correlation_heatmaps}
\end{figure*}
    
    \item \textbf{InfLoRA \cite{Liang2024InfLoRAIL}:}
    Interference-Free Low-Rank Adaptation is a parameter-efficient fine-tuning (PEFT) method proposed for continual learning.

    \item \textbf{SD-LoRA \cite{23}:}
    Decouples the magnitude and direction of LoRA updates to achieve scalable, rehearsal-free class-incremental learning, offering strong stability–plasticity trade-offs with high parameter efficiency.

\end{itemize}

\subsection{Datasets}

Our experiments are conducted on three standard class-incremental learning benchmarks, each chosen to represent a different level of inter-class heterogeneity.

\paragraph{CIFAR-100~\cite{Krizhevsky2009LearningML}:}This dataset comprises 60,000 images across 100 classes and is characterized by a high degree of inter-class heterogeneity,therefore, the four models achieved the best results on this dataset.The related results are shown in Table~\ref{tab:merged_model_num_singlecol}. We follow its standard partition, which allocates 50,000 images to the training set and 10,000 images to the testing set.

\paragraph{ImageNet-R~\cite{Boschini2022TransferWF}:}This dataset contains 30,000 artistic renditions (e.g., paintings, sketches) of 200 ImageNet classes, exhibiting relatively low inter-class heterogeneity. The dataset was partitioned using a 9:1 ratio, resulting in a training set of 27,000 images and a testing set of 3,000 images.

\paragraph{CUB-200~\cite{2011The}:}The Caltech-UCSD Birds-200-2011 (CUB-200) dataset consists of 11,788 images across 200 fine-grained bird species, leading to very low inter-class variance. Following the same 9:1 split methodology, we divided the dataset into 10,610 training images and 1,178 testing images.

\subsection{Semantic Correlation Analysis of Datasets}

\label{appendix:k_selection}

To determine the optimal value of K, we conducted a data-driven analysis on the inherent semantic structures of ImageNet-R, CUB-200, and CIFAR-100.

\textbf{Methodology:} We extracted high-dimensional semantic features for all classes using a pre-trained ResNet-50 and computed the L2-normalized class centroids. The inter-class semantic alignments were measured using cosine similarity. We then applied Agglomerative Hierarchical Clustering with Ward's minimum variance method to uncover the latent structural groupings within the feature space.

\textbf{Data-Driven Observations:} As shown in Figure~\ref{fig:correlation_heatmaps}, we present the correlation heatmaps for the three datasets. To effectively highlight fine-grained inter-class relations, the colormap is dynamically normalized using a Two-Slope normalization centered at the median of off-diagonal correlations. The reordered correlation matrices distinctly exhibit three dominant, highly cohesive semantic macro-clusters (highlighted by the dense red block-diagonal structures).

\textbf{Conclusion:} The visualizations clearly demonstrate that on these three datasets,the data naturally partitions into three major semantic groups. Driven strictly by this inherent data distribution, we established K=3 as the optimal macro-cluster quantity for our experiments.

\subsection{Metrics}
\paragraph{ACC:} It metric measures the overall performance by computing the average accuracy across all $N$ sessions upon the completion of CIL:
\begin{equation}
\text{ACC} = \frac{1}{N} \sum_{k=1}^{N} \text{ACC}(S_k).
\label{eq:acc}
\end{equation}
\paragraph{AAA:}It further accumulates the average accuracy of all encountered sessions after training on each new session. Mathematically, it is defined as the mean of step-wise average accuracies across all training stages:
\begin{equation}
\text{AAA} = \frac{1}{N} \sum_{k=1}^{N} \left( \frac{1}{k} \sum_{i=1}^{k} \text{ACC}_k(S_i) \right).
\label{eq:aaa}
\end{equation}
where $S = \{S_1, S_2, ..., S_N\}$ denotes the complete set of $N$ sessions in class-incremental learning (CIL), $\text{ACC}(S_k)$ represents the accuracy on the $k$-th session $S_k$ after finishing all $N$ sessions, and $\text{ACC}_k(S_i)$ denotes the accuracy on the $i$-th session $S_i$ (1 $\leq$ $i$ $\leq$ $k$) after training on the first $k$ sessions.

\

\section{Theoretical Analysis}
\label{app:proofs_free_energy}

Below, we collect additional assumptions and auxiliary lemmas that serve as technical ingredients for the analysis in Section~\ref{sec:ch3}.

\begin{assumption}
\label{assump:sdlora}
\label{assump:dispersion} 
For any session subset $S\subseteq\{1,\ldots,N\}$,
there exists a trained low-rank approximation $\Delta W_{S}^{[:j_S]}$ such that for any $i\in S$,
\begin{equation}
\left\|A_iB_i-\Delta W_{S}^{[:j_S]}\right\|_{\op}
\ \le\ \epsilon\,\sigma_1(\Delta W_S^\star)\ +\ \varepsilon_S,
\label{eq:sdlora_bound}
\end{equation}
where $\sigma_1(\cdot)$ denotes the largest singular value and
\begin{equation}
\varepsilon_S
:= \min_{X}\ \max_{k\in S}\ \left\|\Delta W_k^\star - X\right\|_{\op}.
\label{eq:cheb_radius}
\end{equation}
Moreover, the accuracy parameter satisfies
\begin{equation}
\epsilon \ \le\ \frac{1}{m+n+r},
\label{eq:epsilon_constraint}
\end{equation}
where $r$ is the LoRA rank and $(m,n)$ are the matrix dimensions.
\end{assumption}

\begin{assumption}
\label{ass:residual_calibration}
Fix a confidence level $1-\delta$.
Let the quantile residual $\eta_\delta$ be defined at level $1-\delta$.
We assume that
\begin{equation}
\eta_\delta < (1-2\epsilon)\,\varepsilon_{1:N}.
\label{eq:residual_calibration}
\end{equation}
\end{assumption}

\begin{lemma}
\label{lem:radius_le_diameter}
For any finite set $\{\Delta W_i^\star\}_{i\in S}$ under $\|\cdot\|_{\op}$, we have
\begin{equation}
\varepsilon_S \le \mathrm{diam}(S),
\end{equation}
where
\begin{equation}
\mathrm{diam}(S) := \max_{p,q\in S} \|\Delta W_p^\star - \Delta W_q^\star\|_{\op}.
\label{eq:diam_def}
\end{equation}
\begin{equation}
\varepsilon_S := \min_{X} \max_{i\in S} \|\Delta W_i^\star - X\|_{\op}.
\label{eq:radius_def}
\end{equation}
\end{lemma}

\begin{proof}
Fix $i_0\in S$ and choose $X = \Delta W_{i_0}^\star$. Then
\begin{equation}
\begin{aligned}
\max_{i\in S} \|\Delta W_i^\star - X\|_{\op}
&= \max_{i\in S} \|\Delta W_i^\star - \Delta W_{i_0}^\star\|_{\op} \\
&\le \max_{p,q\in S} \|\Delta W_p^\star - \Delta W_q^\star\|_{\op} \\
&= \mathrm{diam}(S).
\end{aligned}
\end{equation}
Taking the minimum over centers $X$ yields the desired result.
\end{proof}

\begin{lemma}
\label{lem:width_controls_diam_hp}
Let $S$ satisfy the within-set width condition
\begin{equation}
\forall p,q\in S,\quad |F_p-F_q|\le \tau.
\label{eq:width_condition}
\end{equation}
Under Assumption~\ref{ass:lipschitz_hp}, there exists an event $\mathcal{E}_S$ such that
\begin{equation}
\Pr(\mathcal{E}_S)\ \ge\ 1-\delta_S,
\qquad
\delta_S := \binom{|S|}{2}\,\delta,
\label{eq:ES_prob}
\end{equation}
and on $\mathcal{E}_S$,
\begin{equation}
\mathrm{diam}(S)\ \le\ L_F\tau+\eta_\delta.
\label{eq:diam_bound_on_ES}
\end{equation}
Consequently, on $\mathcal{E}_S$ one has
\begin{equation}
\varepsilon_S \ \le\ \mathrm{diam}(S)
\ \le\ L_F\tau+\eta_\delta.
\label{eq:radius_diam_tau_bound}
\end{equation}
\end{lemma}

\begin{proof}
For each unordered pair $\{p,q\}\subseteq S$, define the event
\begin{equation}
\mathcal{E}_{p,q}
:=
\left\{
\|\Delta W_p^\star-\Delta W_q^\star\|_{\op}
\le
L_F\,|F_p-F_q|+\eta_\delta
\right\}.
\label{eq:Epq_def}
\end{equation}
By Assumption~\ref{ass:lipschitz_hp}, $\Pr(\mathcal{E}_{p,q})\ge 1-\delta$.

Let
\begin{equation}
\mathcal{E}_S
:=
\bigcap_{\{p,q\}\subseteq S}\mathcal{E}_{p,q}.
\label{eq:ES_def}
\end{equation}
By the union bound,
\begin{equation}
\Pr(\mathcal{E}_S)
\ge
1-\sum_{\{p,q\}\subseteq S}\Pr(\mathcal{E}_{p,q}^c)
\ge
1-\binom{|S|}{2}\delta,
\label{eq:union_bound_ES}
\end{equation}
which proves~\eqref{eq:ES_prob}.

On $\mathcal{E}_S$, for all $p,q\in S$,
\begin{equation}
\begin{aligned}
\|\Delta W_p^\star-\Delta W_q^\star\|_{\op}
&\le
L_F\,|F_p-F_q|+\eta_\delta \\
&\le
L_F\tau+\eta_\delta,
\end{aligned}
\label{eq:pairwise_bound}
\end{equation}
where the last inequality uses~\eqref{eq:width_condition}. Taking the maximum over $p,q\in S$ yields
\eqref{eq:diam_bound_on_ES}. Finally, \eqref{eq:radius_diam_tau_bound} follows from
Lemma~\ref{lem:radius_le_diameter}.
\end{proof}

\begin{lemma}
\label{lem:spectral_gap_condition}
Under the assumptions of Lemma~\ref{lem:group_hp},
if
\begin{equation}
L_F \tau + \eta_\delta
<
\varepsilon_{1:N}
+
\epsilon\Bigl(
  \sigma_1\bigl(\Delta W_{1:N}^\star\bigr)
  -
  \max_\ell \sigma_1\bigl(\Delta W_{G_\ell}^\star\bigr)
\Bigr),
\label{eq:spectral_gap_condition}
\end{equation}
then the strict improvement condition~\eqref{tighter_bound} holds; that is, group-wise
training yields a strictly tighter worst-case upper bound than sharing
a single adapter across all sessions.
\end{lemma}

\begin{proof}
By Lemma~\ref{lem:group_hp}, for any group $G_\ell$ we have
\begin{equation}
\begin{aligned}
\max_{i\in G_\ell}
\bigl\|A_i B_i - \Delta W_{G_\ell}^{[:j_\ell]}\bigr\|_{\op}
&\le
\epsilon\,\sigma_1\bigl(\Delta W_{G_\ell}^\star\bigr)
+ \varepsilon_{G_\ell}, \\
\varepsilon_{G_\ell} &\le L_F \tau + \eta_\delta.
\end{aligned}
\label{eq:group_local_bound}
\end{equation}

Hence
\begin{equation}
\begin{aligned}
\lefteqn{
\max_\ell \max_{i\in G_\ell}
\bigl\|A_i B_i - \Delta W_{G_\ell}^{[:j_\ell]}\bigr\|_{\op}
}
\\
\le\ &
\max_\ell
\Bigl(
  \epsilon\,\sigma_1\bigl(\Delta W_{G_\ell}^\star\bigr)
  + \varepsilon_{G_\ell}
\Bigr)
\\
\le\ &
\epsilon\,\max_\ell \sigma_1\bigl(\Delta W_{G_\ell}^\star\bigr)
+ (L_F \tau + \eta_\delta).
\end{aligned}
\label{eq:group_bound_cor}
\end{equation}

On the other hand, without grouping we have
\begin{equation}
\max_{1\le i\le N}
\bigl\|A_i B_i - \Delta W_{1:N}^{[:j]}\bigr\|_{\op}
\le
\epsilon\,\sigma_1\bigl(\Delta W_{1:N}^\star\bigr)
+ \varepsilon_{1:N}.
\label{eq:global_bound_cor}
\end{equation}
If condition~\eqref{eq:spectral_gap_condition} holds, then
\begin{equation}
\epsilon\,\max_\ell \sigma_1\bigl(\Delta W_{G_\ell}^\star\bigr)
+ (L_F \tau + \eta_\delta)
<
\epsilon\,\sigma_1\bigl(\Delta W_{1:N}^\star\bigr)
+ \varepsilon_{1:N}.
\end{equation}
Combining \eqref{eq:group_bound_cor} and \eqref{eq:global_bound_cor},
we see that the grouped worst-case bound is strictly smaller than the
ungrouped worst-case bound, which is exactly condition~\eqref{tighter_bound}.
This completes the proof.
\end{proof}

\begin{lemma}
\label{lem:spectral_gap_lower}
Under the spectral norm, we always have
\begin{equation}
\begin{aligned}
\sigma_{1}\!\left(\Delta W_{1:N}^{\star}\right)
-\max_{\ell}\sigma_{1}\!\left(\Delta W_{G_{\ell}}^{\star}\right)
&\ge
-\varepsilon_{1:N}-\max_{\ell}\varepsilon_{G_{\ell}} \\
&\ge
-2\varepsilon_{1:N}.
\end{aligned}
\label{eq:spectral_gap_lower}
\end{equation}

\end{lemma}

\begin{proof}
Recall that $\sigma_{1}(W)=\|W\|_{\op}$.
For any group $G_{\ell}$ and any $i\in G_{\ell}$, by the triangle inequality,
\begin{equation}
\begin{aligned}
\|\Delta W_{G_{\ell}}^{\star}\|_{\op}
&\le
\|\Delta W_{i}^{\star}\|_{\op}
+
\|\Delta W_{i}^{\star}-\Delta W_{G_{\ell}}^{\star}\|_{\op} \\
&\le
\max_{k\in G_{\ell}}\|\Delta W_{k}^{\star}\|_{\op}
+
\varepsilon_{G_{\ell}}.
\end{aligned}
\label{eq:tri_upper_group}
\end{equation}

Taking the maximum over $\ell$ on both sides yields
\begin{equation}
\max_{\ell}\left\|\Delta W_{G_{\ell}}^{\star}\right\|_{\op}
\le
\max_{1\le k\le N}\left\|\Delta W_{k}^{\star}\right\|_{\op}
+\max_{\ell}\varepsilon_{G_{\ell}}.
\label{eq:max_over_groups}
\end{equation}
Therefore,
\begin{equation}
\begin{aligned}
\lefteqn{
\left\|\Delta W_{1:N}^{\star}\right\|_{\op}
-\max_{\ell}\left\|\Delta W_{G_{\ell}}^{\star}\right\|_{\op}
}
\\
\ge\ &
\left\|\Delta W_{1:N}^{\star}\right\|_{\op}
-\max_{1\le k\le N}\left\|\Delta W_{k}^{\star}\right\|_{\op}
\\
&-\max_{\ell}\varepsilon_{G_{\ell}}.
\end{aligned}
\label{eq:gap_decompose}
\end{equation}

For any $k$, the reverse triangle inequality gives
\begin{equation}
\left\|\Delta W_{1:N}^{\star}\right\|_{\op}
-\left\|\Delta W_{k}^{\star}\right\|_{\op}
\ge
-\left\|\Delta W_{1:N}^{\star}-\Delta W_{k}^{\star}\right\|_{\op},
\label{eq:reverse_tri}
\end{equation}
hence
\begin{equation}
\begin{aligned}
\lefteqn{
\left\|\Delta W_{1:N}^{\star}\right\|_{\op}
-\max_{1\le k\le N}\left\|\Delta W_{k}^{\star}\right\|_{\op}
}
\\
\ge\ &
-\max_{1\le k\le N}\left\|\Delta W_{1:N}^{\star}-\Delta W_{k}^{\star}\right\|_{\op}
\\
=\ &
-\varepsilon_{1:N}.
\end{aligned}
\label{eq:center_point_gap}
\end{equation}

Substituting \eqref{eq:center_point_gap} into \eqref{eq:gap_decompose} yields
\begin{equation}
\left\|\Delta W_{1:N}^{\star}\right\|_{\op}
-\max_{\ell}\left\|\Delta W_{G_{\ell}}^{\star}\right\|_{\op}
\ge
-\varepsilon_{1:N}-\max_{\ell}\varepsilon_{G_{\ell}}.
\label{eq:gap_lb_1}
\end{equation}
Finally, since $G_{\ell}\subseteq\{1,\ldots,N\}$, we have $\varepsilon_{G_{\ell}}\le \varepsilon_{1:N}$ for all $\ell$,
and thus
\begin{equation}
\left\|\Delta W_{1:N}^{\star}\right\|_{\op}
-\max_{\ell}\left\|\Delta W_{G_{\ell}}^{\star}\right\|_{\op}
\ge
-2\varepsilon_{1:N}.
\label{eq:gap_lb_2}
\end{equation}
Converting $\|\cdot\|_{\op}$ back to $\sigma_{1}(\cdot)$ completes the proof.
\end{proof}

Below, we provide the omitted proofs from Sections ~\ref{sec:ch3}.

\subsection{Proof of Lemma~\ref{lem:group_hp}}
\begin{proof}
Fix any group index $\ell$.
Applying Assumption~\ref{assump:sdlora} with $S=G_\ell$ yields that for every $i\in G_\ell$,
\begin{equation}
\bigl\|A_iB_i-\Delta W_{G_\ell}^{[:j_\ell]}\bigr\|_{\op}
\le
\epsilon\,\sigma_1\!\bigl(\Delta W_{G_\ell}^{\star}\bigr)
+\varepsilon_{G_\ell}.
\end{equation}
Taking the maximum over $i\in G_\ell$ gives the first inequality in~\eqref{lemma16}.

Next, by the width condition and Lemma~\ref{lem:width_controls_diam_hp}
(applied to $S=G_\ell$), on the event $\mathcal{E}_{G_\ell}$ we have
\begin{equation}
\varepsilon_{G_\ell}\le L_F\tau+\eta_\delta,
\end{equation}
which implies the second inequality in~\eqref{lemma16}.

Similarly, applying Assumption~\ref{assump:sdlora} with $S=\{1,\ldots,N\}$ and taking the maximum over
$1\le i\le N$ yields the corresponding global bound.
\end{proof}

\subsection{Proof of Theorem~\ref{mainthm}}
\begin{proof}
By the lower bound in Lemma~\ref{lem:spectral_gap_lower}, see \eqref{eq:spectral_gap_lower}, we have
\begin{equation}
\begin{aligned}
\varepsilon_{1:N}
+ \epsilon\Bigl(\sigma_1(\Delta W_{1:N}^\star)
- \max_{\ell}\sigma_1(\Delta W_{G_\ell}^\star)\Bigr)
&\ge \varepsilon_{1:N} - 2\epsilon\,\varepsilon_{1:N} \\
&= (1-2\epsilon)\,\varepsilon_{1:N}.
\end{aligned}
\end{equation}
If condition \eqref{thm8} holds, then it necessarily follows that
\begin{equation}
L_F \tau + \eta_\delta
<
\varepsilon_{1:N}
+ \epsilon\Bigl(\sigma_1(\Delta W_{1:N}^\star)
- \max_{\ell}\sigma_1(\Delta W_{G_\ell}^\star)\Bigr),
\end{equation}
which coincides with the requirement of Lemma~\ref{lem:spectral_gap_condition}, namely condition~\eqref{eq:spectral_gap_condition}.
Therefore the strict improvement condition~\eqref{tighter_bound} is satisfied, which
completes the proof.
\end{proof}

\section{Full Experimental Results}
\label{sec:full_results}

In this section, we present additional experimental results, including analyses of division-collaboration and parameters.

\subsection{Analysis of Division and Collaboration}

 We further investigate the robustness of SDC against varying task granularities by testing it under two distinct continual learning scenarios on the ImageNet-R dataset, a short sequence with 5 sessions ($N=5$) and a highly fragmented long sequence with 20 sessions ($N=20$). These tests explicitly demonstrate how the synergy of principled division and structured collaboration scales across different curriculum pacings.

As shown in Table~\ref{tab:imagenet-r_perf_N5}, under the 5-session configuration, SDC achieves state-of-the-art performance with 87.24\% ACC and 89.59\% AAA. Notably, it outperforms the strongest baseline, SD-LoRA, by a substantial margin of 8.09 points in ACC and 6.58 points in AAA. This leap indicates that when the task stream is partitioned into fewer, larger chunks containing more overlapping semantic concepts, SDC effectively isolates the heterogeneous optimization directions while ensuring immediate and seamless knowledge integration among experts.

\begin{table}[htbp]
  \centering
  \setlength{\tabcolsep}{10pt}
  \small  
  \begin{tabular}{l c@{\,}l c@{\,}l}
    \toprule
    \textbf{Method} & ACC ($\uparrow$) & & AAA ($\uparrow$) & \\
    \midrule
    Fine-Tuning               & 64.92 & & 75.57 & \\
    L2P                       & 73.04 & & 76.94 & \\
    DualPrompt                & 69.99 & & 72.24 & \\
    CODA-Prompt               & 76.63 & & 80.30 & \\
    HiDe-Prompt               & 74.77 & & 78.15 & \\
    InfLoRA                   & 76.95 & & 81.81 & \\
    SD-LoRA                   & \underline{79.15} & & \underline{83.01} & \\
    \textbf{SDC}              & \textbf{87.24} & \textbf{$_{(+8.09)}$} & \textbf{89.59} & \textbf{$_{(+6.58)}$} \\
    \bottomrule
  \end{tabular}
  \caption{Performance of different methods on the ImageNet-R (number of session=5). The 1$^\mathrm{st}$ and 2$^\mathrm{nd}$ best results are highlighted in \textbf{bold} and \underline{underlined}, respectively.}
  \label{tab:imagenet-r_perf_N5}
\end{table}

Table~\ref{tab:imagenet-r_perf_N20} presents the results for the challenging 20-session setting. As the number of sequential sessions increases, the risk of catastrophic forgetting naturally escalates due to the extended chain of weight updates, which is clearly evidenced by the severe performance drop in naive Fine-Tuning (plunging to 49.95\% ACC). Despite this extreme sequence length, SDC demonstrates exceptional resilience, achieving 79.23\% ACC and 83.25\% AAA. It maintains a decisive absolute lead of 3.97 points and 3.03 points over SD-LoRA, respectively.

\begin{table}[htbp]
  \centering
  \setlength{\tabcolsep}{10pt}
  \small  
  \begin{tabular}{l c@{\,}l c@{\,}l}
    \toprule
    \textbf{Method} & ACC ($\uparrow$) & & AAA ($\uparrow$) & \\
    \midrule
    Fine-Tuning               & 49.95 & & 65.32 & \\
    L2P                       & 68.97 & & 74.16 & \\
    DualPrompt                & 65.23 & & 71.30 & \\
    CODA-Prompt               & 69.38 & & 73.95 & \\
    HiDe-Prompt               & 73.59 & & 77.93 & \\
    InfLoRA                   & 69.89 & & 76.68 & \\
    SD-LoRA                   & \underline{75.26} & & \underline{80.22} & \\
    \textbf{SDC}              & \textbf{79.23} & \textbf{$_{(+3.97)}$} & \textbf{83.25} & \textbf{$_{(+3.03)}$} \\
    \bottomrule
  \end{tabular}
  \caption{Performance of different methods on the ImageNet-R (number of session=20). The 1$^\mathrm{st}$ and 2$^\mathrm{nd}$ best results are highlighted in \textbf{bold} and \underline{underlined}, respectively.}
  \label{tab:imagenet-r_perf_N20}
\end{table}

 Table~\ref{tab:merged_model_num_singlecol} investigates how the number of expert models affects the specialization--complementarity trade-off in SDC. Increasing the number of models from two to three consistently improves both ACC and AAA across all datasets: the gains are 1.97/1.59 points on ImageNet-R, 1.72/0.71 points on CIFAR-100, and 0.99/0.62 points on CUB-200, respectively. This indicates that two experts provide insufficient capacity to isolate the heterogeneous optimization directions induced by different incremental sessions. Introducing a third expert enables a more effective division of labor, allowing semantically compatible sessions to be grouped together while preserving sufficient interactions among experts. The effect of adding a fourth model is dataset dependent. CIFAR-100 continues to benefit from a larger expert pool, achieving its best performance with four models, whereas ImageNet-R and CUB-200 exhibit slight degradation relative to the three-model setting. In particular, SDC-M3 achieves the best results on ImageNet-R and CUB-200, suggesting that excessive division can fragment related sessions across multiple experts and weaken the complementary knowledge sharing that is needed for robust continual learning. Thus, the optimal number of experts is determined by the intrinsic semantic diversity and distributional complexity of the session stream, rather than by a fixed architectural preference. Overall, these observations show that SDC should employ enough experts to separate incompatible sessions, while avoiding unnecessary fragmentation of compatible knowledge.

\begin{table}[htbp]
  \centering
  \small  
  \setlength{\tabcolsep}{2pt} 
  \begin{tabular}{lcccccc}
    \toprule
    Method & \multicolumn{2}{c}{ImageNet-R} & \multicolumn{2}{c}{CIFAR-100} & \multicolumn{2}{c}{CUB-200} \\
    \cmidrule(lr){2-3} \cmidrule(lr){4-5} \cmidrule(lr){6-7}
           & ACC($\uparrow$) & AAA($\uparrow$) & ACC($\uparrow$) & AAA($\uparrow$) & ACC($\uparrow$) & AAA($\uparrow$) \\
    \midrule
    SDC-M2 & 83.99 & 87.02 & 93.36 & 95.87 & 83.94 & 90.26 \\
    SDC-M3 & 85.96 & 88.61 & 95.08 & 96.58 & 84.93 & 90.88 \\
    SDC-M4 & 85.85 & 88.29 & 96.13 & 97.40 & 84.80 & 89.50 \\
    \bottomrule
  \end{tabular}
  \caption{Effect of the number of models on SDC performance. }
  \label{tab:merged_model_num_singlecol}
\end{table}

\begin{table*}[htbp]
  \centering
  \begin{tabular}{l c c c c c c} 
    \toprule
    \multirow{2}{*}{\textbf{Division strategy}} & \multicolumn{2}{c}{\textbf{ImageNet-R}} & \multicolumn{2}{c}{\textbf{CIFAR-100}} & \multicolumn{2}{c}{\textbf{CUB-200}} \\
    \cmidrule(lr){2-3} \cmidrule(lr){4-5} \cmidrule(lr){6-7}
                                                & ACC ($\uparrow$) & AAA ($\uparrow$) & ACC ($\uparrow$) & AAA ($\uparrow$) & ACC ($\uparrow$) & AAA ($\uparrow$) \\
    \midrule
    Uniform random division            & \underline{84.58} & \underline{87.04} & \underline{92.34} & \underline{94.75} & \underline{82.79} & \underline{88.15} \\
    Imbalanced division                & 83.06 & 86.12 & 91.64 & 94.72 & 82.13 & 87.83 \\
    Division-aware collaboration       & \textbf{85.96} & \textbf{88.61} & \textbf{95.08} & \textbf{96.58} & \textbf{84.93} & \textbf{90.88} \\
    \bottomrule
  \end{tabular}
  \caption{Effect of division strategies in SDC on multiple datasets. The 1$^\mathrm{st}$ and 2$^\mathrm{nd}$ best results are highlighted in \textbf{bold} and \underline{underlined}, respectively.}
  \label{tab:division_analysis_multi}
\end{table*}

\subsection{Ablation Study}
\label{ablation study}

As shown in Table~\ref{tab:division_analysis_multi}, we evaluate the impact of different task division strategies across the ImageNet-R, CIFAR-100, and CUB-200 datasets. The experimental results demonstrate a consistent performance hierarchy across all benchmarks. Our proposed division-aware collaboration scheme, optimized via Helmholtz free energy, universally achieves the highest ACC and AAA scores. For example, on CIFAR-100, it outperforms the uniform random division strategy by 2.74 points in ACC and 1.83 points in AAA. While uniform random division provides a reasonable baseline by evenly distributing the workload, it ignores the underlying semantic relationships between incremental sessions, leading to suboptimal expert specialization. Conversely, the imbalanced allocation strategy simulated using an extreme 1:9 distribution ratio,yields yields the poorest results across the board. This severe performance degradation occurs because it heavily overloads a single expert while starving the remaining capacity of meaningful optimization opportunities. Ultimately, these findings confirm that the Helmholtz free energy formulation successfully groups semantically compatible tasks, striking an optimal balance between parameter utilization and expert specialization.

\begin{table}[htbp]
  \centering
  \small 
  \setlength{\tabcolsep}{1mm} 
  \begin{tabular}{l c c c c c c c c}
    \toprule
    \multirow{2}{*}{\textbf{Method}} & \multirow{2}{*}{\textbf{HFES}} & \multirow{2}{*}{\textbf{DME}} & \multicolumn{2}{c}{\textbf{ImageNet-R}} & \multicolumn{2}{c}{\textbf{CIFAR-100}} & \multicolumn{2}{c}{\textbf{CUB-200}} \\
    \cmidrule(lr){4-5} \cmidrule(lr){6-7} \cmidrule(lr){8-9}
           &      &     & ACC & AAA & ACC & AAA & ACC & AAA \\
    \midrule
    \multirow{4}{*}{\textbf{SDC}} 
          & $\times$ & $\times$ & 70.62 & 76.32 & 89.54 & 93.67 & 75.64 & 83.19 \\
          & $\times$ & $\checkmark$ & 79.86 & 84.62 & 89.74 & 93.68 & 78.09 & 84.43 \\
          & $\checkmark$ & $\times$ & \underline{82.15} & \underline{86.15} & \underline{94.68} & \underline{95.98} & \underline{83.97} & \underline{90.08} \\
          & $\checkmark$ & $\checkmark$ & \textbf{85.96} & \textbf{88.61} & \textbf{95.08} & \textbf{96.58} & \textbf{84.93} & \textbf{90.88} \\
    \bottomrule
  \end{tabular}
  \caption{Ablation study of HFES and DME modules across datasets. The 1$^\mathrm{st}$ and 2$^\mathrm{nd}$ best results are highlighted in \textbf{bold} and \underline{underlined}, respectively.}
  \label{tab:ablation_hfes_dme_multi}
\end{table}

 As shown in Table~\ref{tab:ablation_hfes_dme_multi},we further conducted ablation experiments on the CIFAR-100 and CUB-200 datasets to verify the effectiveness of the HFES and DME modules.Experimental results show consistent trends across all three datasets:introducing only the DME module yields moderate performance gains;employing only the HFES module leads to more significant improvements;when both HFES and DME are activated together, the model achieves the best performance on all datasets, demonstrating that they can effectively complement each other and jointly enhance model performance.


\bibliography{aaai2027}

\begin{thebibliography}{31}
\providecommand{\natexlab}[1]{#1}

\bibitem[{Aljundi et~al.(2018)Aljundi, Babiloni, Elhoseiny, Rohrbach, and Tuytelaars}]{8}
Aljundi, R.; Babiloni, F.; Elhoseiny, M.; Rohrbach, M.; and Tuytelaars, T. 2018.
\newblock Memory Aware Synapses: Learning What (not) to Forget.
\newblock In \emph{European Conference on Computer Vision}, 144--161.

\bibitem[{Aljundi, Chakravarty, and Tuytelaars(2017)}]{9}
Aljundi, R.; Chakravarty, P.; and Tuytelaars, T. 2017.
\newblock Expert Gate: Lifelong Learning with a Network of Experts.
\newblock In \emph{Computer Vision and Pattern Recognition}, 7120--7129.

\bibitem[{Boschini et~al.(2022)Boschini, Bonicelli, Porrello, Bellitto, Pennisi, Palazzo, Spampinato, and Calderara}]{Boschini2022TransferWF}
Boschini, M.; Bonicelli, L.; Porrello, A.; Bellitto, G.; Pennisi, M.; Palazzo, S.; Spampinato, C.; and Calderara, S. 2022.
\newblock Transfer Without Forgetting.
\newblock In \emph{European Conference Computer Vision}, volume 13683, 692–709.

\bibitem[{Cui et~al.(2025)Cui, Liu, Yu, Huang, and Hong}]{26}
Cui, Y.; Liu, L.; Yu, Z.; Huang, G.; and Hong, X. 2025.
\newblock Few-Shot Audio-Visual Class-Incremental Learning with Temporal Prompting and Regularization.
\newblock In \emph{AAAI Conference on Artificial Intelligence}, volume~39, 16118--16126.

\bibitem[{Durkheim(1997)}]{durkheim1997division}
Durkheim, {\'E}. 1997.
\newblock \emph{The Division of Labor in Society}.
\newblock New York: Free Press.

\bibitem[{Hong, Kim, and Kim(2025)}]{24}
Hong, K.; Kim, G.-h.; and Kim, E. 2025.
\newblock RainbowPrompt: Diversity-Enhanced Prompt-Evolving for Continual Learning.
\newblock In \emph{IEEE/CVF International Conference on Computer Vision}, 1130--1140.

\bibitem[{Hu et~al.(2022)Hu, Shen, Wallis, Allen-Zhu, Li, Wang, Wang, and Chen}]{Hu2021LoRALA}
Hu, E.~J.; Shen, Y.; Wallis, P.; Allen-Zhu, Z.; Li, Y.; Wang, S.; Wang, L.; and Chen, W. 2022.
\newblock Lo{RA}: Low-Rank Adaptation of Large Language Models.
\newblock In \emph{International Conference on Learning Representations}, 1--20.

\bibitem[{Hu et~al.(2019)Hu, Lin, Liu, Tao, Tao, Ma, Zhao, and Yan}]{3}
Hu, W.; Lin, Z.; Liu, B.; Tao, C.; Tao, Z.; Ma, J.; Zhao, D.; and Yan, R. 2019.
\newblock Overcoming Catastrophic Forgetting via Model Adaptation.
\newblock In \emph{International Conference on Learning Representations}, 1--13.

\bibitem[{Krizhevsky and Hinton(2009)}]{Krizhevsky2009LearningML}
Krizhevsky, A.; and Hinton, G. 2009.
\newblock Learning Multiple Layers of Features from Tiny Images.
\newblock In \emph{Handbook of Systemic Autoimmune Diseases}, volume~1, 4--64.

\bibitem[{Lee et~al.(2025)Lee, Lee, chen Chiu, and Tsai}]{27}
Lee, Y.~L.; Lee, C.-Y.; chen Chiu, W.; and Tsai, Y.-H. 2025.
\newblock Exemplar Masking for Multimodal Incremental Learning.
\newblock In \emph{IEEE/CVF Conference on Computer Vision and Pattern Recognition}, 2942--2951.

\bibitem[{Li et~al.(2025{\natexlab{a}})Li, Chen, Shao, Chen, Hong, Qi, and Sun}]{6}
Li, D.; Chen, Z.; Shao, M.; Chen, X.; Hong, S.; Qi, J.; and Sun, H. 2025{\natexlab{a}}.
\newblock Non-Exemplar Class-Incremental Learning via Prototype Correction and Hierarchical Regularization for Specific Emitter Identification.
\newblock \emph{IEEE Transactions on Intelligent Transportation Systems}, 26(8): 12632--12646.

\bibitem[{Li et~al.(2025{\natexlab{b}})Li, Zeng, Dai, and Suganthan}]{10}
Li, D.; Zeng, Z.; Dai, W.; and Suganthan, P.~N. 2025{\natexlab{b}}.
\newblock Complementary Learning Subnetworks Towards Parameter-Efficient Class-Incremental Learning.
\newblock \emph{IEEE Transactions on Knowledge and Data Engineering}, 37(6): 3240--3252.

\bibitem[{Li et~al.(2025{\natexlab{c}})Li, Wang, Zhu, Lin, Yao, and Hu}]{li2025graphs}
Li, J.; Wang, Y.; Zhu, P.; Lin, W.; Yao, X.; and Hu, Q. 2025{\natexlab{c}}.
\newblock Graphs Help Graphs: Multi-Agent Graph Socialized Learning.
\newblock In \emph{Advances in Neural Information Processing Systems}, 1--28.

\bibitem[{Liang and Li(2024)}]{Liang2024InfLoRAIL}
Liang, Y.-S.; and Li, W.-J. 2024.
\newblock InfLoRA: Interference-Free Low-Rank Adaptation for Continual Learning.
\newblock In \emph{IEEE/CVF Conference on Computer Vision and Pattern Recognition}, 23638--23647.

\bibitem[{Liu and Chang(2025)}]{25}
Liu, X.; and Chang, X. 2025.
\newblock LoRA Subtraction for Drift-Resistant Space in Exemplar-Free Continual Learning.
\newblock In \emph{IEEE/CVF Conference on Computer Vision and Pattern Recognition}, 15308--15318.

\bibitem[{Shin et~al.(2017)Shin, Lee, Kim, and Kim}]{5}
Shin, H.; Lee, J.~K.; Kim, J.; and Kim, J. 2017.
\newblock Continual Learning with Deep Generative Replay.
\newblock In \emph{Advances in Neural Information Processing Systems}, 2990--2999.

\bibitem[{Shuttleworth et~al.(2025)Shuttleworth, Andreas, Torralba, and Sharma}]{Shuttleworth2024LoRAVF}
Shuttleworth, R.; Andreas, J.; Torralba, A.; and Sharma, P. 2025.
\newblock LoRA vs Full Fine-tuning: An Illusion of Equivalence.
\newblock In \emph{Advances in Neural Information Processing Systems}, 1--24.

\bibitem[{Smith et~al.(2023)Smith, Karlinsky, Gutta, Cascante-Bonilla, Kim, Arbelle, Panda, Feris, and Kira}]{Smith2022CODAPromptCD}
Smith, J.~S.; Karlinsky, L.; Gutta, V.; Cascante-Bonilla, P.; Kim, D.; Arbelle, A.; Panda, R.; Feris, R.~S.; and Kira, Z. 2023.
\newblock CODA-Prompt: COntinual Decomposed Attention-Based Prompting for Rehearsal-Free Continual Learning.
\newblock \emph{In IEEE/CVF Conference on Computer Vision and Pattern Recognition}, 11909--11919.

\bibitem[{Wah et~al.(2011)Wah, Branson, Welinder, Perona, and Belongie}]{2011The}
Wah, C.; Branson, S.; Welinder, P.; Perona, P.; and Belongie, S. 2011.
\newblock The Caltech-UCSD Birds-200-2011 Dataset.
\newblock \emph{California Institute of Technology}.

\bibitem[{Wang et~al.(2023{\natexlab{a}})Wang, Xie, Zhang, Huang, Su, and Zhu}]{16}
Wang, L.; Xie, J.; Zhang, X.; Huang, M.; Su, H.; and Zhu, J. 2023{\natexlab{a}}.
\newblock Hierarchical Decomposition of Prompt-Based Continual Learning: Rethinking Obscured Sub-optimality.
\newblock In \emph{Advances in Neural Information Processing Systems}, volume~36, 69054--69076.

\bibitem[{Wang et~al.(2023{\natexlab{b}})Wang, Liu, Ji, Wang, Wu, Jiang, Chao, Han, Wang, Shao, and Zeng}]{12}
Wang, Z.; Liu, Y.; Ji, T.; Wang, X.; Wu, Y.; Jiang, C.; Chao, Y.; Han, Z.; Wang, L.; Shao, X.; and Zeng, W. 2023{\natexlab{b}}.
\newblock Rehearsal-free Continual Language Learning via Efficient Parameter Isolation.
\newblock In \emph{Association for Computational Linguistics}, 10933--10946.

\bibitem[{Wang et~al.(2022{\natexlab{a}})Wang, Zhang, Ebrahimi, Sun, Zhang, Lee, Ren, Su, Perot, Dy, and Pfister}]{18}
Wang, Z.; Zhang, Z.; Ebrahimi, S.; Sun, R.; Zhang, H.; Lee, C.-Y.; Ren, X.; Su, G.; Perot, V.; Dy, J.~G.; and Pfister, T. 2022{\natexlab{a}}.
\newblock DualPrompt: Complementary Prompting for Rehearsal-Free Continual Learning.
\newblock \emph{European Conference on Computer Vision}, 13686: 631--648.

\bibitem[{Wang et~al.(2022{\natexlab{b}})Wang, Zhang, Lee, Zhang, Sun, Ren, Su, Perot, Dy, and Pfister}]{Wang2021LearningTP}
Wang, Z.; Zhang, Z.; Lee, C.-Y.; Zhang, H.; Sun, R.; Ren, X.; Su, G.; Perot, V.; Dy, J.~G.; and Pfister, T. 2022{\natexlab{b}}.
\newblock Learning to Prompt for Continual Learning.
\newblock In \emph{IEEE/CVF Conference on Computer Vision and Pattern Recognition}, 139--149.

\bibitem[{Wu et~al.(2025)Wu, Piao, Huang, Wang, Li, Pfister, Meng, Ma, and Wei}]{23}
Wu, Y.; Piao, H.; Huang, L.-K.; Wang, R.; Li, W.; Pfister, H.; Meng, D.; Ma, K.; and Wei, Y. 2025.
\newblock {SD}-LoRA: Scalable Decoupled Low-Rank Adaptation for Class Incremental Learning.
\newblock In \emph{International Conference on Learning Representations}, 1--17.

\bibitem[{Yang et~al.(2023)Yang, Lin, Zhang, Liu, Liang, Ji, and Ye}]{13}
Yang, B.; Lin, M.; Zhang, Y.; Liu, B.; Liang, X.; Ji, R.; and Ye, Q. 2023.
\newblock Dynamic Support Network for Few-Shot Class Incremental Learning.
\newblock \emph{IEEE Transactions on Pattern Analysis and Machine Intelligence}, 45(3): 2945--2951.

\bibitem[{Yao et~al.(2024)Yao, Wang, Zhu, Lin, Li, Li, and Hu}]{pmlr-v235-yao24d}
Yao, X.; Wang, Y.; Zhu, P.; Lin, W.; Li, J.; Li, W.; and Hu, Q. 2024.
\newblock Socialized Learning: Making Each Other Better Through Multi-Agent Collaboration.
\newblock In \emph{International Conference on Machine Learning}, volume 235, 56927--56945.

\bibitem[{Yao et~al.(2025)Yao, Wang, Zhu, Lin, Zhao, Guo, Li, and Hu}]{yao2025socialized}
Yao, X.; Wang, Y.; Zhu, P.; Lin, W.; Zhao, R.; Guo, Z.; Li, W.; and Hu, Q. 2025.
\newblock Socialized Coevolution: Advancing a Better World through Cross-Task Collaboration.
\newblock In \emph{International Conference on Machine Learning}, volume 267, 71780--71797.

\bibitem[{Zenke, Poole, and Ganguli(2017)}]{7}
Zenke, F.; Poole, B.; and Ganguli, S. 2017.
\newblock Continual Learning Through Synaptic Intelligence.
\newblock In \emph{International Conference on Machine Learning}, volume~70, 3987--3995.

\bibitem[{Zhang et~al.(2025)Zhang, Liu, Silv\'en, Pietik\"ainen, and Hu}]{zhang2025few}
Zhang, J.; Liu, L.; Silv\'en, O.; Pietik\"ainen, M.; and Hu, D. 2025.
\newblock Few-shot class-incremental learning for classification and object detection: A survey.
\newblock \emph{IEEE Transactions on Pattern Analysis and Machine Intelligence}, 47(4): 2924--2945.

\bibitem[{Zhang et~al.(2024)Zhang, Xu, Zhang, Liu, Hooi, and Deng}]{Zhang2023ExploringCM}
Zhang, J.; Xu, X.; Zhang, N.; Liu, R.; Hooi, B.; and Deng, S. 2024.
\newblock Exploring Collaboration Mechanisms for LLM Agents: A Social Psychology View.
\newblock In \emph{In Association for Computational Linguistics}, 14544--14607.

\bibitem[{Zhou et~al.(2024)Zhou, Wang, Qi, Ye, chuan Zhan, and Liu}]{Zhou2023ClassIncrementalLA}
Zhou, D.-W.; Wang, Q.; Qi, Z.; Ye, H.-J.; chuan Zhan, D.; and Liu, Z. 2024.
\newblock Class-Incremental Learning: A Survey.
\newblock \emph{IEEE Transactions on Pattern Analysis and Machine Intelligence}, 46(12): 9851--9873.

\end{thebibliography}

\end{document}